\documentclass[10pt]{article} 
\usepackage[preprint]{tmlr}

\usepackage{amsmath,amsfonts,bm}

\def\eqref#1{equation~\ref{#1}}

\def\1{\bm{1}}

\DeclareMathAlphabet{\mathsfit}{\encodingdefault}{\sfdefault}{m}{sl}
\SetMathAlphabet{\mathsfit}{bold}{\encodingdefault}{\sfdefault}{bx}{n}

\newcommand{\R}{\mathbb{R}}

\usepackage[utf8]{inputenc} 
\usepackage[T1]{fontenc}    
\usepackage{url}            
\usepackage{booktabs}       
\usepackage{amsfonts}       
\usepackage{nicefrac}       
\usepackage{microtype}      
\usepackage{xcolor}         

\usepackage{amssymb}
\usepackage{mathtools}
\usepackage{enumitem}
\usepackage{wrapfig}
\usepackage{mathrsfs}
\usepackage{amsmath}
\usepackage{amsthm}
\usepackage{hyperref}
\usepackage{cleveref}
\usepackage{placeins}

\newtheorem{theorem}{Theorem}[section]
\newtheorem{proposition*}{Proposition}
\newtheorem{proposition}[theorem]{Proposition}
\newtheorem{lemma}[theorem]{Lemma}
\newtheorem{corollary}[theorem]{Corollary}
\theoremstyle{definition}
\newtheorem{definition}[theorem]{Definition}

\theoremstyle{remark}
\newtheorem{remark}[theorem]{Remark}

\let\oldTheta\Theta
\def\Theta{\mathit{\oldTheta}}
\DeclareMathOperator{\rank}{rank}
\newcommand{\param}{\Theta}
\newcommand{\layer}{W}
\newcommand{\localF}{\LimitLoss_{loc}}
\DeclareMathOperator{\grad}{\nabla}
\newcommand{\dataset}{\mathcal{X}}

\newcommand{\intermdata}{\mathcal{Z}}
\DeclareMathOperator{\Loss}{\mathcal{L}}
\DeclareMathOperator{\loss}{\ell}

\DeclareMathOperator{\TCV}{\mathrm{TCV}}
\DeclareMathOperator{\DNC}{\mathrm{DNC1}}
\DeclareMathOperator{\WCSS}{\mathrm{WCSS}}
\newcommand{\matrices}{\mathcal{M}}

\newcommand{\LimitLoss}{\mathscr L}

\usepackage{xcolor}

\title{Provable Emergence of Deep Neural Collapse and Low-Rank Bias in $L^2$-Regularized Nonlinear Networks}

\author{\name Emanuele Zangrando \email emanuele.zangrando@gssi.it \\
      \addr Gran Sasso Science Institute
      \AND
      \name Piero Deidda \email piero.deidda@gssi.it \\
      \addr Gran Sasso Science Institute
      \AND
      \name Simone Brugiapaglia \email simone.brugiapaglia@concordia.ca\\
      \addr Concordia University 
      \AND
      \name Nicola Guglielmi \email nicola.guglielmi@gssi.it \\
      \addr Gran Sasso Science Institute
      \AND
      \name Francesco Tudisco \email f.tudisco@ed.ac.uk \\
      \addr University of Edinburgh
      \AND
      }

\def\month{MM}  
\def\year{YYYY} 
\def\openreview{\url{https://openreview.net/forum?id=XXXX}} 

\begin{document}

\maketitle

\begin{abstract}
We present a unified theoretical framework connecting the first property of Deep Neural
Collapse (DNC1) to the emergence of implicit low-rank bias in nonlinear networks trained
with $L^2$ weight decay regularization. Our main contributions are threefold. First, we derive a
quantitative relation between the Total Cluster Variation (TCV) of intermediate embeddings
and the numerical rank of stationary weight matrices.
In particular, we establish that, at any critical point, the
distance from a weight matrix to the set of rank-$K$ matrices is bounded by a constant
times the TCV of earlier-layer features, scaled inversely with the weight-decay parameter.
Second, we prove global optimality of DNC1 in a constrained representation-cost setting for
both feedforward and residual architectures, showing that zero TCV across intermediate
layers minimizes the representation cost under natural architectural constraints.
Third, we establish a benign landscape property: for almost every interpolating
initialization there exists a continuous, loss-decreasing path from the initialization to
a globally optimal, DNC1-satisfying configuration. 
Our theoretical claims are
validated empirically; numerical experiments confirm the predicted relations among TCV,
singular-value structure, and weight decay. These results indicate that neural collapse and
low-rank bias are intimately linked phenomena arising from the optimization geometry
induced by weight decay.
\end{abstract}

\section{Introduction}
Recent years have seen growing theoretical and experimental evidence of low-rank properties of deep neural networks. For example, it has been noted that a large portion of singular values from linear layers in different models can be removed almost without affecting performance, both for feedforward and residual networks \citep{Schotthofer_low_rank,Ydelbayev_low_rank}. Analogous results have been shown on feedforward and convolutional networks for image classification \citep{Li_basis_CNN,wang2021pufferfish,khodak2021initialization,Schotthofer_low_rank} as well as on Recurrent Neural Networks (RNNs) trained to learn neural dynamics in neuroscience tasks \citep{schuessler2020interplay,pellegrino2023low}. 

Alongside the wealth of experimental evidence of low-rank properties of neural networks, several recent works have investigated the low-rank bias phenomenon from the theoretical point of view,  providing evidence of decaying singular values in deep linear architectures \citep{arora2019implicit,feng2022rank,huh2022low} 
and of low-rank properties of networks' Hessian maps \citep{singh2021analytic}, analyzing low-rank geometric properties of training gradient flow dynamics for deep linear networks \citep{bah2022learning}, showing that low-rank bias is related to weight decay \citep{galanti2023characterizing,sukenik2024neural,kuzborskij2025lowrankbiasweightdecay}, and to training invariances \citep{le2022training,ziyin2025parametersymmetrybreakingrestoration}. In particular, the authors relate the rank of the weight matrices during SGD to the batch size adopted during the training of the network.
However, based only on the batch size, the results from \citep{galanti2023characterizing}  become uninformative when the batch size is comparable to the number of parameters of the neural network. 

A related phenomenon commonly observed during the training of classifiers is ``neural collapse'' \citep{papyan2020prevalence, zhu2021geometric} and more generally "deep (or intermediate) neural collapse" \cite{parker2023neural, rangamani2023feature, sukenik2024deep, sukenik2024neural}. This phenomenon is rigorously described in the literature using four properties named \textbf{NC1--4}, the first of which (\textbf{NC1}) essentially states that the within-class variability of intermediate representations goes to zero even for layers in the middle of a deep neural network 
(see \Cref{def:DNC1}). 
   The optimality of neural collapse at the last layer has been extensively investigated and proved in different settings, otherwise only few extensions have been proposed for its deep version.

In particular, the emergence of the neural collapse phenomenon at the last layer was first empirically observed in \citep{papyan2020prevalence}, and proven to be optimal in different settings,
such as Unconstrained Feature Model (UFM) \citep{Fang_minority_collapse,mixon2020neuralcollapseunconstrainedfeatures,zhu2021geometric}, with constrained UFM \citep{Dang_relu,yaras2023neuralcollapsenormalizedfeatures}, with cross-entropy loss \citep{e2021emergencesimplexsymmetryfinal,kunin2023asymmetricmaximummarginbias,Wanli_imbalanced_data}, Neural Tangent Kernel (NTK) block structure aligned with class labels \citep{seleznova2023neural}, linear tail and wide architecture \citep{jacot2025wide}, effects of the loss function \citep{zhou2022loss}, and in the context of Graph Neural Networks (GNNs) \citep{kothapalli2023neural}.
Moreover, some works analyzed the dynamics leading to neural collapse, from the observation of a central path dynamics \citep{han2022neuralcollapsemseloss}, to the use of large network dynamics \citep{jacot2025wide}. Finally, some authors characterize stability of the stationary points of UFM in the regularized setting \citep{ji2022unconstrainedlayerpeeledperspectiveneural,zhu2021geometric,zhou2022optimizationlandscapeneuralcollapse}, and in the Riemannian setting, with feature and weights constrained to oblique manifolds \citep{yaras2023neuralcollapsenormalizedfeatures}.

The analysis of the neural collapse phenomenon beyond the last layer was first empirically studied by \citep{rangamani2023feature}, and addressed under the name of ``deep'' (or ``intermediate") neural collapse phenomenon ($\textbf{DNC}$), where notably the NC1 property is observed to hold even at intermediate layers.
Proving $\DNC$ optimality has proven more challenging than NC because it requires switching from the linear setting of the last layer to the nonlinear one of the intermediate layers. A theoretical analysis of this  phenomenon has been conducted in \citep{sukenik2024deep,sukenik2024neural} that establish $\DNC$ optimality in the Deep-UFM model but restrict to binary classification or require specific architectural assumptions (L-DUFM, pyramidal widths). In particular, the key limitation of prior work about $\DNC$ is that it studies only simplified linear/UFM models, not able to take into account the full extent of architectural constraints.

\subsection{Main Contributions}
In this work we show, for the first time, a rigorous connection between the phenomenon of deep neural collapse and low-rank implicit bias for general nonlinear networks. The main contribution is theoretical and, in particular, our main results, \emph{which hold in the presence of weight decay}, can be summarized as follows: 
\begin{itemize}[leftmargin=*,itemsep=0pt,topsep=0pt]
    \item in \Cref{sec:main_results} we prove that within-class variability of hidden layers during training (including the initial data set) has a quantifiable effect on the numerical rank of the network parameters as the latter decays with the total cluster variation of the hidden representations (\Cref{Theorem_best_rank_k-approx_bound});
    \item in \Cref{sec:main_DNC} we show that zero total cluster variation for all intermediate layers is optimal in a fairly general setting, \emph{both for feedforward and residual networks} (\Cref{thm:main_DNC_optimal_ff}); 
    \item \Cref{thm:main_reachability_nc} establishes that for (almost) every interpolating initial condition, there exists a loss-decreasing trajectory connecting the initial condition to a collapsed configuration, thereby offering an explanation to the consistent empirical observation of deep neural collapse \citep{rangamani2023feature}.
\end{itemize}
Overall, the combination of the results above proves that in the regime of small weight decay regularization, zero within-class variability for all intermediate representations is globally optimal under some natural constraints, therefore showing that global minima in parameter space are low-rank matrices. Moreover, convergence to these highly structured solutions may be regarded as universal in the sense that, \emph{there is essentially no loss barrier between interpolating features to deep-neural collapsed ones}, which constitutes a step forward towards explaining the universality in the emergence of (deep) neural collapse.

Our paper improves upon the results presented in the related literature in different aspects. 
First, we show that the rank of stationary points in parameter space is controlled by $\DNC$, a metric defining the first property of neural collapse.  In particular, our result also holds for stationary points that do not satisfy DNC2/DNC3. Notably, the results in \citep{sukenik2024neural} suggest that there exist settings where our results become crucial, indeed the authors show that, in such settings, while $\DNC$ keeps being optimal, the other properties may fail.

Second, the setting in which we prove our results is not a simplified model, but an exact representation of a feedforward/ResNet architecture with two additional constraints. The analysis is based on the extrinsic view of a neural network architecture as a constraint in the space of representable intermediate embeddings
\citep{jacot2022featurelearningl2regularizeddnns,han2022neuralcollapsemseloss}. Finally, in this generalized setting, we prove a \emph{benign} property of the loss landscape, which could explain the frequent emergence of neural collapse in intermediate representations as observed in \citep{rangamani2023feature}.

The remainder of the paper is organized as follows. In \Cref{sec:preliminaries} we
introduce notation and set up the optimization framework. In \Cref{sec:main_results} we prove
the rank–TCV bound and discusses its implications. \Cref{sec:main_DNC} contains
the optimality and landscape results, and \Cref{sec:num-exp} validates the
theory empirically.
%
\section{Setting and Preliminaries}\label{sec:preliminaries}
Consider a dataset $\mathcal{X} = \{x_j\}_{j=1}^N \subseteq \mathbb{R}^d$ with $c$ corresponding labels $\mathcal{Y}=\{y_j\}_{j=1}^N\subseteq \mathbb{R}^c$. 

We assume to train a fully connected $L$-layer {feedforward neural network} $f(\param,x)$ of the form 
\begin{equation}
\label{DEF_Neural_Network}
\begin{aligned}
z^0=x, \quad f(\param,x) = z^{L}, \quad   z^{l} = \sigma_l\big(\layer_l z^{l-1} + b_l\big), \; l=1,\dots,L 
\end{aligned}
\end{equation}
to learn from the given data and labels.
In equation \eqref{DEF_Neural_Network}, $\sigma_l$ are  scalar nonlinear activation functions acting entrywise and $\param$ denotes all trainable parameters $\param =(W_1,\dots,W_L,b_1,\dots,b_L)$. We write
\begin{equation}\label{DEF_Trainable_parameters}
    \begin{aligned}
         \param_{l_1}^{l_2} &:=(\layer_{l_1+1},\dots,\layer_{l_2}, b_{l_1+1},\dots, b_{l_2}), \quad l_1<l_2,\\
         f_{l_1}^{l_2}(\Theta_{l_1}^{l_2}, z^{l_1}) &:= \sigma(W_{l_2}\cdots \sigma(W_{l_1+1}z^{l_1}+b_{l_1+1})+\cdots+b_{l_2})
    \end{aligned}
\end{equation}
to denote the training parameters and concatenation of layers from $l_1+1$ to $l_2$, respectively. We adopt the convention that, when we do not specify the lower index then it starts from zero, and when we omit the top one it ends at the overall depth $L$, namely, $\param^{l_2} := \param_0^{l_2}$, $f^{l_2}:=f^{l_2}_0$ and $\param_{l_1} := \param_{l_1}^L$, $f_{l_1}:=f^{L}_{l_1}$. Notice that, with this notation, $\param= \param_0 = \param^L = \param_0^L$, $f=f_0=f^L=f_0^L$ and $\param^0$ is empty. Finally, for the sake of brevity, we write $\intermdata^l:=\{z^l_i\}_{x_i\in \dataset}$ to denote the set of outputs of the $l$-th layer, where $z^l_i = f^l(\param^l, x_i)$
is defined as in \eqref{DEF_Neural_Network}  with input data $x_i$, and  $\intermdata^0=\dataset$.

Moreover, to simplify notation, we will denote the feature matrices of input, output, and intermediate layers as
\begin{align*}
&X=[x_1,\dots, x_N]\in\mathbb R^{d \times N}, \quad Y=[y_1,\dots, y_N]\in\mathbb R^{c \times N}, \quad  Z^l = [z_1^l,\dots,z_N^l]  \in \mathbb R^{n_l \times N},  
\end{align*}
where $1\leq l \leq L$.
The parameters $\param$ are assumed to be trained by minimizing a regularized loss function $\Loss_\lambda(\param)$ composed by the sum of losses at individual data points, i.e.,
\begin{equation}\label{Def_Loss_function}
\Loss_\lambda(\param):=\frac{1}{N} \sum_{i=1}^N \loss\big(f(\param,x_i), y_i\big) + \lambda \|\param\|^2,\quad \text{where}\quad \|\Theta \|^2 = \sum_{i = 1}^L \Bigl(\|W_i \|^2+ \|b_i \|^2\Bigr).
\end{equation}
Here $\|\cdot\|$ denotes the entrywise $L^2$ norm (when the input is a matrix, it coincides with the Frobenius norm) and  $\lambda \geq 0$ is a \emph{weight decay} or \emph{regularization} parameter. 

In this setting, in \Cref{sec:main_results} we will quantify a precise relationship between the numerical rank of the stationary points of $\mathcal L_\lambda$ and the total cluster variation of intermediate feature matrices.
%
\section{A Link Between Low-Rank and Deep-Neural Collapse}\label{sec:main_results}
%
%
Note that any critical point of $\mathcal{L}_{\lambda}(\Theta)$ satisfies the following equation: 
\begin{equation}\label{Eq_critical_point_eq_regularized_loss}
\lambda \layer_l^*=-\nabla_{\layer_l}\Loss_0(\param^*) \qquad \forall l=1,\dots, L.
\end{equation}
The intuition is that Equation \eqref{Eq_critical_point_eq_regularized_loss} resembles an ‘‘eigenvalue equation'' for the nonlinear operator $\nabla \Loss_0$.  
In other words, for a fixed weight decay parameter $\lambda$, training the neural network corresponds to finding a nonlinear eigenvector $\Theta^*$ having $-\lambda$ as eigenvalue.
In addition, equation \eqref{Eq_critical_point_eq_regularized_loss} immediately implies that whenever the operators $\nabla_{W_l} \Loss_0$ have low rank (i.e., they take values in some low-rank space), then also any weight matrix $W^*$ composing the eigenvector $\Theta^*$ has necessarily low rank. In the following, we investigate this relationship more in depth, providing sufficient conditions to have approximately low-rank gradients.

The first key point in this study is provided by the following proposition. Indeed, for any differentiable loss function $\loss$, the gradient with respect to any layer, at a fixed data point, has rank at most 1 (see also Lemma~3.1 in \citep{galanti2023characterizing}). Here we consider the gradient at a batch of data points. The proof of this result can be found in \Cref{Proof_Prop_rank_of_lossgradient_at_single_data}.
\begin{proposition}[Small batches yield low-rank gradients]
\label{Prop_rank_of_lossgradient_at_single_data}
Let $f(\param,x)$ be a feedforward neural network as in \eqref{DEF_Neural_Network} and $\loss:\mathbb{R}^c \times \mathbb{R}^c \to \mathbb{R}$ be a differentiable loss function as in \eqref{Def_Loss_function}. Then, for any $\{(x_i,y_i)\}_{i=1}^K\subseteq \mathbb{R}^d\times \mathbb{R}^c$, we have
$\rank\Big(\grad_{\layer_l} \sum_{i=1}^K \loss\big(f(\param,x_i), y_i\big)\Big)\leq K$.
\end{proposition}
Note that, given $z^j=f^j(\param^j,x)$, as long as $j<l$, it holds 
\begin{equation}\label{eq_Intermediate_gradients_stay_invariant}
    \grad_{W_l} \loss\big(f(\param,x), y\big)=
    \grad_{\layer_l} \loss \big(f_j(\Theta_j, z^j),y\big).
\end{equation}
Hence, fixed a batch of data, \Cref{Prop_rank_of_lossgradient_at_single_data} allows us to control the rank of the gradient, with respect to a subsequent layer, of any intermediate loss function $\loss$.

In addition, \Cref{Prop_rank_of_lossgradient_at_single_data} immediately yields an example where gradients have necessarily low rank. 
Consider the case of degenerate data, i.e., assume that there exist $K$ pairs $\{(\bar{x}_k, \bar{y}_k)\}_{k=1}^K$ such that any data point $(x_i,y_i)$ is equal to some $(\bar{x}_k, \bar{y}_k)$. 
Then, by \Cref{Prop_rank_of_lossgradient_at_single_data}, $\grad_{\layer_l}\sum_{i}\loss(\param,x_i,y_i)$ has rank at most $K$.
But the right-hand side of \eqref{Eq_critical_point_eq_regularized_loss} has rank at most $K$, hence any critical point of the loss function $\Loss_{\lambda}$ necessarily corresponds to a neural network having all layers of rank at most $K$. 
The same argument can be extended to the situation of total intermediate neural collapse \citep{Poggio_hidden_nc}, where the intermediate hidden set $\intermdata^l$ consists of at most $K$ distinct points.
Indeed, assume that for a stationary point $\param^*$ 
there exists a layer $l$ and points $\{(\bar{z}_k^l,\bar{y}_k)\}_{k=1}^K$ such that any pair $(z^l_i,y_i)=(\bar{z}_k^l,\bar{y}_k)$ for some $k$. 
Then, by \Cref{Prop_rank_of_lossgradient_at_single_data} and the subsequent remark \eqref{eq_Intermediate_gradients_stay_invariant}, any subsequent layer, $\layer_j^*$ with $j>l$, has rank at most $K$.  

In the following, we prove that under weaker hypotheses, like the collapse of hidden representations, any critical point of the loss function produces layers close to low-rank. 
In particular, the approximate rank of the layers at a stationary point is given by the number of clusters (or collapsing sets) and the approximation error depends both on the regularity of the loss function in the clusters and the within-cluster variability, which we measure as the total cluster variation of each hidden set $\intermdata^l$.
%
\subsection{Clustered Data and Total Cluster Variation}\label{SEC:Clustered_data_total_cluster_variation}
Assume $\mathcal{C}=\{\mathcal{C}_1, \ldots, \mathcal{C}_K\}$ to be a $K$-partition of the input data $\mathcal{X}$ and corresponding labels $\mathcal{Y}$, i.e., for any $(x_i,y_i)\in (\mathcal{X}, \mathcal{Y}) $ there exists a unique $\mathcal{C}_k\in \mathcal{C}$ such that $(x_i,y_i)\in\mathcal{C}_k$. 
We denote  by $N_k = |\mathcal{C}_k|$ the cardinality of the $k$-th family and adopt the notation
%
$(\mathcal{X},\mathcal{Y}) 
= \bigcup_{k=1}^K \bigcup_{i \in \mathcal{C}_k} \{(x_i,y_i)\}
= \bigcup_{k=1}^K \bigcup_{i=1}^{N_k} \{(x_i^k,y_i^k)\}$,
%
where with a small abuse of notation we use $\mathcal{C}_k$ to denote also the set of indices whose corresponding data points are contained in $\mathcal{C}_k$.
Now let $\bar{x}_k=\bar{z}^0_k$, $\bar{y}_k$ and $\bar{z}^l_k$ be the centroids of the sets $\{x_i\}_{i\in \mathcal{C}_k}$, $\{y_i\}_{i\in \mathcal{C}_k}$ and $\{z^l_i\}_{i\in \mathcal{C}_k}$, respectively, i.e., 
\begin{equation*}
    \bar{z}^l_k=\frac{1}{N_K}\sum_{i\in \mathcal{C}_k} z_i^l,
    \qquad \bar{y}_k=\frac{1}{N_K}\sum_{i\in \mathcal{C}_k}y_i.
\end{equation*}
To quantify how clustered the data is with respect to the partition $\mathcal{C}$, we consider the \emph{Within-Cluster Sum of Squares (WCSS)}, defined as 
\begin{equation}\label{DEF_WCSS}
\WCSS_{l}(\mathcal{C})=\frac{1}{N}\sum_{k=1}^{K} \sum_{i\in\mathcal{C}_k}
\|\big(z^l_i,y_i\big)- \big(\bar{z}_k^l,\bar{y}_k\big)\|^2.
\end{equation}
This quantity measures the within-cluster variance (see, e.g., \citep{hastie2009elements}) 
of the data at layer $l$ with respect to a prescribed partition and the corresponding centroids. 

Note that having a clustered dataset at the $l$-th layer means that there exists a suitable partition $\mathcal{C}^*$ such that $\WCSS_{l}(\mathcal{C}^*)$ is sufficiently small. To measure the within-cluster variability of the data at the $l$-th layer when it is partitioned into at most $K$ clusters, we introduce the \emph{Total Cluster Variation}.
\begin{definition}[Total Cluster Variation]
 The $K$-th \emph{Total Cluster Variation (TCV)} at layer $l$, is defined as 
\begin{equation}\label{DEF_TCV}
\TCV_{K,l}:=\min_{\mathcal{C}\in \cup_{r=1}^K\mathcal{P}_r} \WCSS_{l}(\mathcal{C}),
\end{equation}
where $\mathcal{P}_r$ is the set of all partitions of $\{(z_i^l,y_i)\}_{i=1}^N$  into $r$ sets. 
\end{definition}

Note that, for $K$ sufficiently smaller than $N$, a small value of  $\TCV_{K,l}$  indicates an intermediate collapse of the network's features at layer $l$. Moreover, being $\TCV_{K,l}$ defined through a variational problem to get the tightest bound possible, it is trivially upper bounded by the total cluster variation on the classes. 
These considerations are directly connected with the notion of deep (or intermediate) neural collapse \citep{sukenik2024deep,papyan2020traces,rangamani2023feature}, whose first characterizing property is recalled below. 
\begin{definition}[Deep Neural Collapse 1]\label{Def_DNC1}
\label{def:DNC1}
    We say that a neural network $f$ trained on a dataset $(\mathcal X, \mathcal Y)$ satisfies $\DNC$ at layer $l$ if $\TCV_{K,l} = 0$ for $K = c$ and the minimizer in \eqref{DEF_TCV} is given by the partition corresponding to the classes.
\end{definition}
%
\subsection{Neural Collapse Yields Low-Rank Bias}
%
In this section, we show that the occurrence of well-clustered hidden representations, i.e., of intermediate neural collapse, necessarily yields low-rank properties for the underlying neural network in the presence of weight decay.
Given a partition $\mathcal{C}=\{\mathcal{C}_k\}_{k=1}^K$ as in \Cref{SEC:Clustered_data_total_cluster_variation}, we can introduce the centroid-based loss at the $l$-th layer:
\begin{equation*}
\label{DEF_controid_based_loss}
\Loss^{l,\mathcal{C}}_\lambda(\param) = \sum_{k = 1}^K \pi_k \loss\big(f_l(\param_{l},\bar{z}_k^l), \bar{y}_k\big) + \lambda \|\param\|^2, \quad \pi_k = \frac{N_k}{N}.
\end{equation*}
Note that $\mathcal{L}^{0,\mathcal{C}}_\lambda$ is equivalent to the standard loss $\Loss_\lambda$ in the situation of degenerate data, i.e., when the data within any family of $\mathcal{C}$ coincides with its centroid.
Moreover, from the previous discussion and \Cref{Prop_rank_of_lossgradient_at_single_data}, we know that the gradient with respect to any of the last $L-l$ layers of the $l$-th centroid-based loss has rank at most $K$.
 In the following lemma (proved in \Cref{Lemma3.3Proof}), an easy computation allows us to compare the gradients of the original loss and the centroid-based loss:
\begin{lemma}[Centroid-based gradient approximation of full gradient]\label{Lemma_approx__fullgradient_by_centroid_based_loss_gradient}
    Fix a partition $\mathcal C=\{\mathcal C_k\}_{k=1}^K$ and a layer $l$ with weight matrix $W_l$. Let the function 
    $\loss(f_j(\Theta_j,z), y)$ be $C^1$ in $\Theta_j$ and $\nabla_{W_l}\loss(f_j(\Theta_j,z), y)$ be $C^2$ in $(z,y)$
    for any $j<l$. 
    Then, $\forall j<l$
    \begin{equation*}
        \begin{array}{lr}
\grad_{\layer_l}\Loss_0(\param)=\nabla_{\layer_l}\Loss_0^{j,\mathcal{C}}(\param)+R(\param), \quad \|R(\param)\|\leq M_{l,j}(\param,\mathcal{C}) \WCSS_{j}(\mathcal{C})
    \end{array}
    \end{equation*}
     where $\displaystyle{M_{l,j}(\param,\mathcal{C}):=\sup_{\substack{k=1,\dots,K  \\ (z,y)\in 
    \mathcal{H}^j_k}} 
    \|\nabla^2_{(z,y)}\nabla_{\layer_l} \loss(f_j(\param_{j},z),y)\|}$
    and $\mathcal{H}^j_k:=\mathrm{ConvHull}\{(z_i^j,y_i)\}_{i\in\mathcal{C}_k}$.
\end{lemma}

\begin{proof}[Proof (sketch)]Here we provide a sketch of the proof and we refer to Appendix~\ref{Lemma3.3Proof} for a detailed argument.
    Fixed a single data point $(z_i^j,y_i)$ with $i\in\mathcal{C}_k$, we consider the Taylor expansion in $(\bar{z}_k^j,\bar{y}_k)$ of $\nabla_{\layer_l}\ell(f_j(\param_{j},z_i^j),y_i)$ with Lagrange remainder.
    Finally, we note that the first-order contributions sum to zero because of the definitions of the centroids, yielding the desired expressions.
\end{proof}

As a consequence of the above lemma, for any stationary point, we can control the distance of the $l$-th layer from the set of matrices having rank less than or equal to $K$ in terms of the $K$-th total cluster variation of the first $l-1$ intermediate outputs.
Before stating our main result, we define the set of $d_1\times d_2$ matrices having rank equal to $K$, $\matrices_K^{d_1,d_2}:=\{A\in \R^{d_1\times d_2}|\;\rank(A)= K\}$.
We will omit the dimension-related superscripts when clear from the context.

\begin{theorem}[Small within-class variability yields low-rank bias]\label{Theorem_best_rank_k-approx_bound}
    Assume $\param^*$ to be a stationary point of $\Loss_\lambda(\Theta)$. Then, for any $l=1,\dots,L$ and $K\leq N$,
    \begin{equation*}   
     \min_{Z\in \cup_{r=1}^K\matrices_r}
    \|\layer^*_l - Z\|  \leq  \min_{j=1,\dots,l-1}
    \frac{M_{l,j}(\param^*)}{\lambda}\mathrm{TCV}_{K,j},
    \end{equation*}
    where $
    M_{l,j}(\param^*) := \!\!\!\!\!\displaystyle{ \sup_{(\alpha, \beta)\in \mathcal H^j} \!\!\|\nabla_{(z_j,y)}^2\nabla_{\layer_l}\\ \loss(f_j(\param^{*}_j,z^j), \alpha, \beta)\|}$
    and $\mathcal H^j:=\mathrm{ConvHull}\{(z^j_i,y_i)\}_{i=1}^N$.
\end{theorem}
\begin{proof}[Proof (sketch)]
Observe that from \eqref{Eq_critical_point_eq_regularized_loss}, at the stationary point we have 
    \begin{equation}\label{Thmeq_critical_point_eq}
        \layer^*_l=-\frac{1}{ \lambda}\grad_{\layer_l}\Loss_0(\param^*).
    \end{equation}
    Now, Lemma \ref{Lemma_approx__fullgradient_by_centroid_based_loss_gradient} allows us to control the distance of the gradient in the right-hand side of \eqref{Thmeq_critical_point_eq} from $\matrices_K^{d_{l-1},d_l}$. In particular, given $\mathcal{C}\in \cup_{r=1}^K\mathcal{P}_r$, for any $j=1,\dots,l-1$ we see that
    \begin{equation*}
    \begin{aligned}
    \min_{Z\in \cup_{r=1}^K\matrices_r} \|\layer^*_l - Z\|
    &\leq \frac{\|\grad_{W_l} \Loss_0(\param^*)- \grad_{W_l} \Loss_0^{j,\mathcal{C}}(\param^*)\|}{\lambda}
    \leq\frac{M_{l,j}(\param^*,\mathcal{C})\WCSS_{j}(\mathcal{C})}{\lambda},
    \end{aligned}
    \end{equation*}
    where we have used that $\grad_{W_l} \Loss_0^{j,\mathcal{C}}(\param^*)/\lambda$ has rank smaller than or equal to the number of families of $\mathcal{C}$.
    Then, first minimizing over $j=1,\dots,l-1$, second bounding from above $M_{l,j}(\param^*,\mathcal{C})\leq M_{l,j}(\param^*)$ for any $j$, and third minimizing over all the possible partitions $\mathcal{C}\in  \cup_{r=1}^K\mathcal{P}_r$ yield the thesis.
\end{proof}

\Cref{Theorem_best_rank_k-approx_bound} shows that an intermediate neural collapse, \Cref{def:DNC1}, combined with weight decay regularization leads to approximately low-rank network layers, where the numerical ranks effectively depend on the $\TCV$ of intermediate layers. Its proof can be found in \Cref{Theorem3.4Proof}. 

Notice that the vanishing property of $\TCV$ in intermediate layers ($\DNC$) has already been studied for simplified models in, e.g.,  \citep{Poggio_hidden_nc,sukenik2024deep,sukenik2024neural,jacot2025wide}. We therefore notice here the special role of $\DNC$ with respect to the other $\mathrm{DNC}$ properties commonly studied, such as the ones presented in \citep{sukenik2024deep}: the first property alone implies that the stationary points in parameter space must have low-rank, as predicted by \Cref{Theorem_best_rank_k-approx_bound}.

In the next section, 
we will prove that, in a fairly general setting, $\DNC$ is optimal for all intermediate layers: this means that under a set of natural constraints, a zero total cluster variation for all intermediate layers is globally optimal in the small $\lambda$ regime. This shows, in combination with \Cref{Theorem_best_rank_k-approx_bound}, that all global minima in this regime are low rank. Moreover, we will show that for almost every initial condition, there is essentially no loss barrier to configurations with $\TCV = 0$ for all $l=1,\dots,L$, shedding new light on the emergence deep neural collapse and, in turn, on the emergence of low-rank bias. 
%
%
%
\section{Vanishing $\TCV$: Global Optimality of Deep-Neural Collapse}\label{sec:main_DNC}
We start by defining the \emph{representation cost} functional for a biasless network with $L^2$ regularization in all layers but the first one:
\begin{equation}\label{eq:representation_cost_1}
    \LimitLoss(\Theta) := \begin{cases}
    +\infty, \quad \text{if}\,\, f(\Theta;X) \ne Y, \\
    \frac{1}{2}\sum_{l = 2}^{L-1} \|W_{l+1} \|^2, \quad \text{otherwise}.
\end{cases}
\end{equation}

This functional has been studied in a variety of different settings, both linear \citep{Ranzato_2021} and nonlinear \citep{jacot2023bottleneck,jacot2023implicit}.
While in the deep linear case the explicit form of the minima of $\LimitLoss$ are known to be Shatten quasi-norms \citep{Ranzato_2021, ongie2022role}, in the general case a full characterization is still missing. In \citep{jacot2023implicit,jacot2023bottleneck}, the authors were able to characterize, in the large depth regime, how the minimal value essentially depends on the data $y = f^*(x)$ one is interpolating, essentially showing that the first order term scales with $L$ and it is a function bounded between the Jacobian and the bottleneck rank of the function $f^*$.

The general intuition is that $\LimitLoss(\Theta)$ represents the loss of interpolating networks in the small $\lambda$ regime, in fact from \citep{Scagliotti_2022}, we have that $\lambda^{-1}\mathcal L_\lambda \underset{\lambda \to 0}{\to} \LimitLoss(\Theta)$ in the $\Gamma$-convergence sense (\Cref{def:Gamma_convergence}), meaning that minimizers of $\mathcal{L}_{\lambda}$ converge to minimizers of $\LimitLoss$ as $\lambda$ goes to zero, see \Cref{prop:convergence_of_minimizers}.
Now, using the observation done in \citep{jacot2022featurelearningl2regularizeddnns}, we can reparameterize the weight matrices $\{W_l\}$ using the feature matrices in the following way:
\begin{equation}\label{reformulated_weight_matrices}
Z^{l+1} = \sigma_{l+1}(W_{l+1}Z^l) \iff W_{l+1} = \sigma_{l+1}^{-1}(Z^{l+1})Z^{l,+},
\end{equation}
where $Z^{l,+}$ denotes the Moore-Penrose pseudoinverse of the feature matrix $Z^l$. If we want to reformulate the representation cost function in \eqref{eq:representation_cost_1} based on the equivalence discussed in \eqref{reformulated_weight_matrices} we have to additionally take into account for a constraint relating the different intermediate representations, literarily $Z^{l}$ and $Z^{l+1}$ have to satisfy the following equation
\begin{equation}\label{constraintI}
    \sigma_{l+1}^{-1}(Z^{l+1})Z^{l,+} Z^l = \sigma_{l+1}^{-1}(Z^{l+1}) \qquad \forall l=0,\dots, L-1.
\end{equation}
Note that the constraint in \eqref{constraintI}
is effectively a representability condition on the features, equivalent to $W_{l+1}Z^{l} = \sigma_{l+1}^{-1}(Z^{l+1})$.
Given this new representation of the weight matrices, we can reformulate the representation cost as a function of the feature matrices $Z^l$, namely:
\begin{equation}\label{eq:representation_cost_2}
    \LimitLoss(Z) := \begin{cases}
    +\infty, \quad \text{if}\,\, Z^L \ne Y, \\
    \frac{1}{2}\sum_{l = 1}^{L-1} \|\sigma_{l+1}^{-1}(Z^{l+1})Z^{l,+}\|^2,\quad \text{otherwise}.
\end{cases}
\end{equation}
As proven in \citep{jacot2022featurelearningl2regularizeddnns}, it is not restrictive to study this reparametrized representation cost in the sense that the local minimizers of \eqref{eq:representation_cost_1}, subject to the representability constraints \ref{constraintI}, are in bijection with those of \eqref{eq:representation_cost_2}, therefore making them equivalent. The key observation is that the objective in \eqref{eq:representation_cost_2} is local, i.e., it is a sum of local objectives $\LimitLoss_l = \LimitLoss_l(Z^{l+1},Z^l)$ depending just on neighbour representations. Therefore, since $\|W_1 \|^2$ is not present in the regularization, $\LimitLoss$ depends on $Z^1$ just through $\LimitLoss_1(Z^2,Z^1) = \frac12\|\sigma_2^{-1}(Z^2)Z^{1,+} \|^2$, which can be minimized explicitly with respect to $Z^1$,  as a function of the next representation $Z^2$. For this reason, the lack of the first layer in the regularization term is the key to obtaining the results in the following section, as it makes it possible to study recursively the structure of minima layer by layer.
To study it, we consider the additional hypothesis that $Y$ is in one-hot encoding form, i.e. $y_j=e_{i_j}$ for any $j$, where $\{e_i\}$ is the canonical basis of $\mathbb{\R}^c$.
We also point out that, without considering additional constraints, the minima of $\LimitLoss$ may not exist. 
To avoid these situations, we study the minima of $\LimitLoss$ subject to additional constraints.

\subsection{Characterizing Optimal Representations: Natural Architectural Constraints}
Before stating \Cref{thm:main_DNC_optimal_ff}, we introduce three sets of constraints that capture essential properties of 
realizable representations in actual neural networks. These constraints are a way to encode properties of the underlying architecture in the space of representable intermediate embeddings. In particular, this allows one to view dually the neural network under consideration as a pure constraint on the space of intermediate representations, which forces the embedding of one layer $Z^l$ to be linked to that of the embedding before $Z^{l-1}$ through the imposed architecture. We first start with the definition of the constraints.
\begin{definition}[Constraints] \label{eq:constraint}
    Consider the set $\mathcal{S} = \mathcal S(Z^1)$ consisting of all networks $Z = (Z^2,\dots,Z^L)\in \mathbb R^{n_2 \times N} \times \dots \times \mathbb R^{n_L \times N}$ that satisfy the following conditions \ref{cond1}-\ref{cond2}-\ref{cond3}:
\begin{enumerate}[label=(\Roman*), noitemsep, topsep=0pt, leftmargin=*]
    \item\label{cond1} $\sigma_{l+1}^{-1}(Z^{l+1})Z^{l,+} Z^l = \sigma_{l+1}^{-1}(Z^{l+1}), \quad \forall l=1,\dots, L-1$;
    \item\label{cond2} $s_i(Z^l) \leq C_{l+1}s_i(\sigma_{l+1}^{-1}(Z^{l+1})),\quad   \forall i=1,\dots, r_l,\quad   \forall l=2,\dots, L-1$,
    \item\label{cond3} $\rank(\sigma_l^{-1}(Z^l)) \geq \rank(Y)$,
\end{enumerate}
where $C_l>0$ are arbitrary, $r_l$ is the rank of $\sigma_{l+1}^{-1}(Z^{l+1})$, $s_i(A)$ denotes the $i$-th singular value of a matrix $A$ and $\sigma_2,\dots,\sigma_L$ is a set of entrywise omeomorphisms with $\sigma_l(x) \ne 0$ for all $x \ne 0$.
\end{definition}
\paragraph{Constraint~\ref{cond1}: representability.} For any realizable sequence of representations $Z^1,\dots, Z^L$, each representation $Z^{l+1}$ must satisfy $Z^{l+1} = \sigma_l(W_l Z^l)$ 
for some weight matrix $W_l$.
When we reparameterize the representation cost functional in \eqref{eq:representation_cost_1} in terms of the features embeddings (see \eqref{eq:representation_cost_2}), the constraint \ref{cond1}, by coupling the embeddings $z^l$ and $Z^{l+1}$, expresses this representability requirement. 
 We remark that this constraint does not restrict the optimization problem further, as it was shown in \citep{jacot2022featurelearningl2regularizeddnns} that the stationary points of \eqref{eq:representation_cost_1} are in bijection with the ones of \eqref{eq:representation_cost_2} with constraint~\ref{cond1}. 
 \paragraph{Constraint~\ref{cond2}: effect of nonlinearity.}
 The second constraint~\ref{cond2} can be directly related to the effect of the nonlinear activations, and it represents exactly the situation for diagonal feature matrices, for which the nonlinearity acts directly on the singular values. The necessity of constraint~\ref{cond2} for the theoretical analysis relies on making it possible to study the alignment of the optimal basis, avoiding any possible singular value escaping to infinity. Interestingly, notice that while in \ref{cond2} $C$ is an arbitrary positive constant, we observe numerically in \Cref{sec:num-exp} that a perfect fit happens for $C = 1$, reminiscent of a norm balancing condition on all the layers.
 \paragraph{Constraint \ref{cond3}: excluding lower-than rank $K$ solutions.}
 The third set \ref{cond3} is introduced here as it arises naturally in the proof of the theorem as a consequence of \ref{cond2}, and it allows us to rule out the presence of global minima of lower rank. This is because, as we will show in the proof in \Cref{sec:ff_global}, all global minima satisfy the equality constraint.  We notice that, as a byproduct of \Cref{thm:main_DNC_optimal_ff}, we know that if we get rid of constraint \ref{cond3}, then all global minimizers either satisfy $\DNC$ or some layer embedding has rank strictly smaller than $K$. We expect these minimizers to be exceptional in some sense, as they were proven to be optimal in a very particular setting in \citep{sukenik2024neural}. We also remark that the singular values structure is also predicted recursively by \Cref{thm:main_DNC_optimal_ff}, which we will numerically show also in \Cref{sec:num-exp}.

We highlight here that constraint \ref{cond2} corresponds to a constraint just on the nonzero singular values of the matrices $Z^l$ and $\sigma_{l+1}^{-1}(Z^{l+1})$. Therefore, \ref{cond2} allows the case in which all feature matrices $Z^l$ for $l=1,\dots,L-1$ are full-rank. Moreover, $\mathcal S$ is bounded but not compact because of constraint (III), making it nontrivial to prove the existence of a minimizer. The following main result on the minimizers of a constrained version of \eqref{eq:representation_cost_2} holds:

\begin{theorem}[$\DNC$ is optimal for constrained representation cost]\label{thm:main_DNC_optimal_ff}
Let $\sigma_2,\dots,\sigma_L$ be a set of entrywise homeomorphisms and $\sigma_1$ an analytic non-polynomial function, with $\sigma_l(x) = 0 $ if and only if $x = 0$ for every $l =1, \ldots, L$. Assume $K = n_L \leq \dots \leq n_{1} = N \leq 2^{n_1-1}$.
Let $\mathcal S \subset \mathbb R^{n_1 \times N} \times \dots \times \mathbb R^{n_L \times N}$ be as in \Cref{eq:constraint}.
Then, for almost every $Z^1$, the global minima of the optimization problem
\begin{equation*}
\min_{Z \in \mathcal S} \LimitLoss(Z) 
\end{equation*}
satisfy $\DNC$ for all intermediate layers $l = 1,\dots,L$. Moreover, any global minimizer $\bar{Z}$ satisfies 
\[
\bar Z^{l} \propto O_l\sigma_{l+1}^{-1}(\bar Z^{l+1}), \quad O_l \in \R^{n_{l} \times n_{l+1}},  \;O_l^\top O_l = I, \quad \forall \,l=2,\dots,L-1.
\]
\end{theorem}
\begin{proof}[Proof (sketch)]
    First, we consider the set of $Z^1 \in \mathbb R^{N \times N}$ such that $\rank(Z^1) = N$, which is a generic condition as shown in  \Cref{lemma:large_layer_lindip}. Now, fixed $Z^1$ satisfying this condition we look at the minimizers in the $Z^2$ variables, subject to the constraints in \Cref{eq:constraint}. Since regularization on the second term $W_2$ is not present, there is just a single term involving $Z^2$, leading to the following:
    \begin{align*}
    & \min_{Z^2 \in \mathbb R^{n_2 \times N}}\|\sigma_3^{-1}(Z^3)Z^{2,+} \|^2\\
    & \qquad \text{s.t.}\quad 
    \sigma_{3}^{-1}(Z^3) Z^{2,+}Z^2 = \sigma_{3}^{-1}(Z^3) \\
    &\qquad\qquad \;\sigma_{2}^{-1}(Z^2)Z^{1,+}Z^1 = \sigma_{2}^{-1}(Z^2) \\
    &\qquad \qquad \; s_{i}(Z^{2}) \leq C_2 s_i(\sigma_{3}^{-1}(Z^3)). 
    \end{align*}
Notice that, since $Z^1$ is invertible, the second representability constraint is trivially satisfied, as $Z^{1,+}Z^1 = I_N$.
Therefore, thanks to \Cref{lemma:ff_lemma_svals}, we can describe exactly the global minimizers of this optimization problem as a function of $\bar Z^2 \propto O_2 \sigma_3^{-1}(Z^3)$. By substituting this expression in $\LimitLoss$, we can do recursively the same for $Z^3$, by using the third constraint on the rank of representations.
\end{proof}

\begin{remark}
For almost every $Z^1$ is intended in the sense of pushforward Lebesgue measure, i.e., for almost every $(X,W_1) \in \R^{n_0 \times N} \times \R^{n_1 \times N}$.
\end{remark}

A rigorous formulation and proof of \Cref{thm:main_DNC_optimal_ff} involves multiple technical lemmas, and can be found in \Cref{sec:ff_global}. 

\Cref{thm:main_DNC_optimal_ff} can be reinterpreted as follows: if we allow the architecture to collapse already at the first layer, that is, if the kernel of $W_1$ can be large enough (i.e., $n_1 \geq N$), then, under \ref{cond1}--\ref{cond3}, it is optimal to collapse all intermediate layers. Notice that the recursive nature of \Cref{thm:main_DNC_optimal_ff} describes all the optimal intermediate embeddings $Z^l$ recursively as a function of $\sigma_{L}^{-1}(Z^L)$, both in terms of singular values and singular vectors. As an example, $Z^{L-1}$ has to be proportional to a rotated version of $\sigma_{L}^{-1}(Y)$, and therefore we know singular values and vectors exactly. Therefore, by going backward recursively, we can recover the structure of all of the previous layers.

We emphasize that \Cref{thm:main_DNC_optimal_ff} \emph{can be extended to residual networks under a weaker set of assumptions}. Indeed, in that case, we can
get rid of all constraints, except the natural one in \ref{cond1}. Given the similarity of the key ideas, we present this result along with its proof in \Cref{sec:resnet_global}.

In the case of feedforward networks, with a stronger request on the intermediate widths, we have the following additional main result, proving that for almost every interpolating set of parameters, there exists a continuous path in parameter space along which the loss decreases, finally converging to globally optimal points, which satisfy $\DNC$ for \Cref{thm:main_DNC_optimal_ff}.
\begin{theorem}[Benign loss landscape: monotonic paths to $\DNC$]\label{thm:main_reachability_nc}
Assume the intermediate widths in $\mathcal{S}$ satisfy the condition $K \leq n_L = n_{L-1}  = \dots = n_1 = N$ and that the $\sigma_l$'s are analytic. Then, for almost every initial condition $Z(0) \in \mathcal S$, there exists a  continuous path $\gamma(t)$ such that:
\begin{enumerate}
    \item $\gamma(0) = Z(0)$ and $\gamma(1)$ is a global minimum of $\LimitLoss$;
    \item $\LimitLoss \circ \gamma$ is continuous Lebesgue-almost everywhere (except for a finite set of times), with all discontinuities $t_1,\dots,t_m$ removable, i.e., $\underset{t \to t_i^-}{\lim} \LimitLoss(\gamma(t)) = \underset{t \to t_i^+}{\lim} \LimitLoss(\gamma(t)) > \LimitLoss(\gamma(t_i)), \quad \forall i = 1,\dots,m;$
    \item for all continuity times $s\geq t$ we have $\LimitLoss(\gamma(t)) \geq \LimitLoss(\gamma(s)).$
\end{enumerate}
\end{theorem}

\begin{proof}[Proof (sketch)]
    As in \Cref{thm:main_DNC_optimal_ff}, we consider the set of $Z^1 \in \mathbb R^{N \times N}$ such that $\rank(Z^1) = N$, which is a generic condition as shown in  \Cref{lemma:large_layer_lindip}. Now, fixed $Z^1$ satisfying this condition we look at the minimizers in the $Z^2$ variables. Thanks to \Cref{lemma:descent_Q}, there exists an energy decreasing path from the initial condition to the global minimizer in $Z^2$. We recursively concatenate this with the optimal path in $Z^3$, where $Z^2$ follows the central path dynamics (staying optimal as in \citep{han2022neuralcollapsemseloss}). At the end, we end up with the concatenation of $L$ loss decreasing paths in $\mathcal S$, which is therefore globally energy decreasing and connects the initial condition to the constrained global minimizers.
\end{proof}

The full proof of \Cref{thm:main_reachability_nc} can be found in \Cref{subsec:ff_reachability}. The main takeaway is that, from almost every interpolating initial condition (aside from a set of measure zero), one can construct a continuous path in parameter space along which the loss decreases (up to a finite set of downward jumps) and the network converges to a global minimum characterized by deep neural collapse.
In particular, \Cref{thm:main_reachability_nc} implies that 
\[
\inf_{\gamma \in \Gamma} \sup_{t \in [0,1]} \LimitLoss(\gamma(t)) \leq \LimitLoss(Z(0)),
\]
where $\Gamma:= \{\gamma \in \mathcal C^0([0,1];\mathcal S) \,|\,: \gamma(0) = Z(0),\gamma(1)\, \text{satisfies}\,\, \DNC\,\, \forall l = 1,\dots,L \}$.

We highlight that in any numerical method, one looks at a discretized version of the path $\gamma$, so the ``probability" of seeing discontinuity points is zero, no matter how small the learning rate. In particular, no matter how small the stepsize, there always exists a sequence of points that decreases the loss and reaches $\DNC$ configurations. This result is exactly formalized in the following corollary, together with a more detailed explanation in \Cref{remark:reachability_finite_step}:
\begin{corollary}[No loss barrier for any finite stepsize]\label{cor:finite_step_reachability}
    In the setting of \Cref{thm:main_reachability_nc}, for almost every $Z(0) \in \mathcal S$ there exists
 a sequence of curves $\gamma_k:[0,1] \to \mathcal S$ uniformly converging to the continuous curve $\gamma$ of \Cref{thm:main_reachability_nc} such that the following hold for all $k \geq 1$:
    \begin{itemize}
        \item $\gamma_k$ is piecewise constant in a set of disjoint intervals $\{ I_{j,k}:=[T_{j,k},T_{j+1,k}) \}_{j=0,\dots,k-1}$ whose union is $[0,1]$, and with $0<T_{j+1,k}-T_{j,k}<\frac{1}{k}$ for all $j = 0,\dots,k-1$;
        \item $\gamma_k(0) = Z(0)$, $\gamma_k(1)\in \underset{Z \in\mathcal S}{\arg\min} \,\LimitLoss(Z)$
        \item $\LimitLoss$ is decreasing, i.e., $\LimitLoss(\gamma_k(0))\geq \LimitLoss(\gamma_k(T_{1,k})))\geq \LimitLoss(\gamma_k(T_{2,k}))) \geq \dots \geq \LimitLoss(\gamma_k(T_{k,k}))) \geq \LimitLoss(\gamma_k(1)).$
    \end{itemize}
\end{corollary}
The idea of the proof \Cref{cor:finite_step_reachability} relies essentially on a discretization of the path from \Cref{thm:main_reachability_nc}, and can be found in \Cref{proof:finite_step_reachability}. \\

A second way to interpret \Cref{thm:main_reachability_nc} and \Cref{cor:finite_step_reachability} is that the loss landscape of $\LimitLoss$, despite being highly nonconvex, is well behaved in the sense that there are essentially no loss barriers from almost all interpolating points to global minima. This result, being independent of the specific temporal dynamics imposed during training, provides theoretical support for the empirical ``universality'' of neural collapse observed in classification tasks when using sufficiently wide architectures \citep{Poggio_hidden_nc} and possibly gives an hint on why it is observed towards the end of training \citep{papyan2020prevalence,Poggio_hidden_nc} (i.e., in the interpolating phase, where the landscape is well behaved).\\
Moreover, the phenomenon is ``unidirectional'', progressing from high to low rank as demonstrated in the proof, which is reminiscent of the behaviour observed in the deep linear case in \citep{wang2024implicit}. 

Finally, we point out that these results are recovered as the singular limit in the regime where an arbitrarily small penalization term for the first layer $\|W_2\|^2= \| \sigma_2^{-1}(Z^2) Z^{1,+} \|^2$ is included. A natural question arises: are the highly structured solutions given by \Cref{thm:main_DNC_optimal_ff} preserved when $\|W_1 \|^2$ is included?
To answer this question about stability, we consider the following perturbed loss, where $\phi$ is a generic nonnegative lower-semicontinuous function defined on the set $\mathcal{S}$ (\Cref{eq:constraint}),  
\begin{equation}\label{eq:loss_regularized_on_1st_layer}
\LimitLoss_{\lambda}(Z) = \begin{cases}
    +\infty, \quad \text{if}\,\,\ Z^L \ne Y, \\
    \frac{1}{2}\sum_{l=1}^{L-1} \|\sigma_{l+1}^{-1}(Z^{l+1})Z^{l,+} \|^2 + \frac{1}{\lambda} \phi(Z), \quad \text{otherwise},
\end{cases}
\end{equation}
and show that, as $\lambda\to\infty$, the global minima of 
$\LimitLoss_{\lambda}$ converge to global minima of $\LimitLoss$.
Note that varying the value of $\lambda$ is equivalent to rescaling the dataset $X$ by a factor $\lambda$.
\begin{theorem}[Stability of the $\DNC$ global minima]\label{prop:main_gammaconv_firstlayer}

Assume $\phi:\mathcal S \to \R_+$ is lower semicontinuous. Then, the family of functionals $\LimitLoss_{\lambda}: \mathcal S \to \R$ converge to $\LimitLoss: \mathcal S \to \R$ as $\lambda \to +\infty$, in the $\Gamma$-convergence sense.
\end{theorem}
Moreover, as a classical consequence of $\Gamma$-convergence, if $Z^* \in \mathcal S$ is an accumulation point for a sequence of minimizers $Z_n \in \arg\min_{Z\in \mathcal{S}} \LimitLoss_{\lambda_n}(Z)$ with $\lambda_n \to +\infty$, then $\LimitLoss(Z^*) = \min_{Z\in \mathcal{S}} \LimitLoss(Z)$. 
The proof is included in \Cref{subsec:ff_reachability}. 
We also highlight here that the existence of an accumulation point $Z^*$ is not immediate, since $\mathcal S$ is not compact. We finally remark that an equivalent stability result holds even for residual networks, as discussed in \Cref{remark:resnet_stability}.
\begin{remark}
\Cref{prop:main_gammaconv_firstlayer} can be interpreted from two complementary perspectives. 
First, the optimality of $\DNC$ configurations established in \Cref{thm:main_DNC_optimal_ff,thm:NC1_optimal_resnet} remains stable under the inclusion of a backbone regularization term. Consequently, one expects the configurations across all layers, up to the first representative layer, to remain close to $\DNC$. In particular, such proximity to $\DNC$ configurations, even when the first-layer weight decay is incorporated, is expected to occur in the regime $\|X\| \gg \|Y\|$.\\
Alternatively, the parameter $\lambda$ in \Cref{prop:main_gammaconv_firstlayer} can also be interpreted as representing the network depth. To formalize this interpretation, define
\[
\LimitLoss^{[l_1,l_2]}_{L}(Z):=
\begin{cases}
+\infty, & \text{if } Z^L \neq Y,\\[3pt]
\frac{1}{L}\displaystyle\sum_{l=l_1}^{l_2}\phi_{l+1}(Z^{l+1},Z^l), & \text{otherwise.}
\end{cases}
\]
For $Z^L = Y$, consider the difference
\[
\LimitLoss^{[0,L-1]}_L - \LimitLoss^{[l^*,L-1]}_L 
= \LimitLoss^{[0,l^*-1]}_L 
= \frac{1}{L}\sum_{l=0}^{l^*-1} \phi_{l+1}(Z^{l+1},Z^l) 
\underset{\Gamma}{\longrightarrow} 0, \quad \text{as } L \to +\infty.
\]
Hence, in the large-depth regime, even in the absence of a pyramidal backbone structure, one expects the minimizers to exactly satisfy the $\DNC$ condition from the first layer $l^*$ for which the set of representable features $Z^{l^*}$ contains the minimizer predicted by \Cref{thm:main_DNC_optimal_ff}, with the special case of the UFM, where the set of representable matrices $Z^{l^*}$ coincides with $\R^{n_{l^*}\times N}$.\\
This observation elucidates that, in general, for sufficiently deep architectures, the structure of the minimizers asymptotically approaches that predicted in \Cref{thm:main_DNC_optimal_ff}, where $Z^2$ is effectively replaced by the minimal layer index $l^*$ such that $Z^{l^*}$ is \emph{representative enough}.
\end{remark}

%
\section{Experimental Evaluation} \label{sec:num-exp}
\subsection{Qualitative Behaviour of \Cref{Theorem_best_rank_k-approx_bound}}
In this section, we show in a simple setting the qualitative dependency of $\min_{Z\in \cup_{r=1}^K\matrices_r} \|\layer^*_l - Z\|$ 
on the regularization parameter $\lambda$ and on the TCV of the dataset. In particular, fixed a set of vectors $\mu_1,\dots,\mu_K$, we consider a random Gaussian mixture dataset $X \in \R^{n \times N}$ with $x_i \sim N(\mu_{y_i},\sigma^2)$, and a $Y \in \R^{K \times N}$ one-hot encoding matrix with $K=10$ classes. In this way, we know that $\TCV_{K,0} \sim O(\sigma^2)$, and we expect, from \Cref{Theorem_best_rank_k-approx_bound}, the ranks to approach $K = 10$ as $\sigma^2 \to 0$. Then, for each $(\lambda,\sigma)$ in a grid of values, we sampled $X,Y$ according to the distribution described above, and we trained a four-layer fully connected neural network with the loss defined in \Cref{Def_Loss_function}. In \Cref{fig:dependency_on_variance} we report the results of this experiment. We can observe that consistently, for each $\lambda$ in our grid, we have that all matrices are essentially of rank $K$ for $\sigma$ small enough. Moreover, we also observe the decreasing behavior with respect to $\lambda$ consistently with \Cref{Theorem_best_rank_k-approx_bound}.
\begin{figure}[t]
\centering \includegraphics[width = .70\textwidth, clip, trim=.18cm .33cm .2cm .25cm]{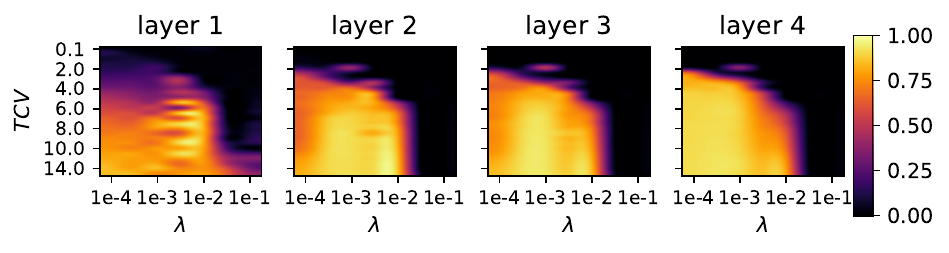}
\caption{For each couple $(\lambda,\sigma)$ we report the relative distance $\sum_{j>K}s_j^2/\sum_{j}s_j^2$ of trained weight matrices from the closest rank-$ K =10$. Here, $s_j$ are the singular values of each weight matrix.}
\label{fig:dependency_on_variance}
\end{figure}
\subsection{Convergence to the Optimal Solution of \Cref{thm:main_reachability_nc}}\label{subsec:convergence_to_svals}
We start by illustrating and numerically validating what we claimed in \Cref{thm:main_reachability_nc,prop:main_gammaconv_firstlayer}. In particular, by setting $\delta = 1/\lambda$ in \eqref{eq:loss_regularized_on_1st_layer} we can show that for $\delta \to 0$, the singular values of all the intermediate post-activation features are the ones described in \Cref{thm:main_DNC_optimal_ff}. To set up the experiment, we subsampled the MNIST dataset \citep{deng2012mnist} with $N = 200$ examples, and we considered a $5$ layer feedforward network with all $\sigma_l = \tanh$ activations and widths equal to $n_1 = \dots = n_4 = N  =200, n_5 = K = 10$. By including $\delta$, the loss becomes
$\mathcal L_{\lambda,\delta} = \frac{1}{2}\|f_{\Theta}(X)-Y \|^2 + \frac{\lambda}{2} \sum_{l=2}^{L} \|W_l \|^2 + \frac{\lambda \delta}{2} \|W_1 \|^2 + \frac{\lambda'}{2}\sum_{l=1}^L \|Z^l \|^2 $,
where the last regularization term can be thought of as a weaker way of imposing the second set of constraints in \Cref{thm:main_reachability_nc}. In particular, for this experiment we considered $\lambda = \lambda'  = 10^{-5}$.\\
We trained the network until convergence (way past interpolation, with a final data loss lower than $10^{-6}$ in all examples) for different values of $\delta$, and we compared the actual singular values of the feature matrices after training with the ones predicted by \Cref{thm:main_DNC_optimal_ff}. The results are shown in \Cref{fig:main_svals_prediction}. As we can notice, the singular values are converging to the ones predicted by \Cref{thm:main_DNC_optimal_ff}, and the convergence with respect to $\delta$ seems to be fast. In fact, already for moderate sizes of $\delta$ (e.g., $\delta = 2$), the singular values are close to the predicted ones. 

\begin{figure}[t]
\centering
\begin{tabular}{ccc}
$\delta = 10$ & $\delta = 2$ & $\delta = 1$ \\ 
\includegraphics[width=0.2\textwidth]{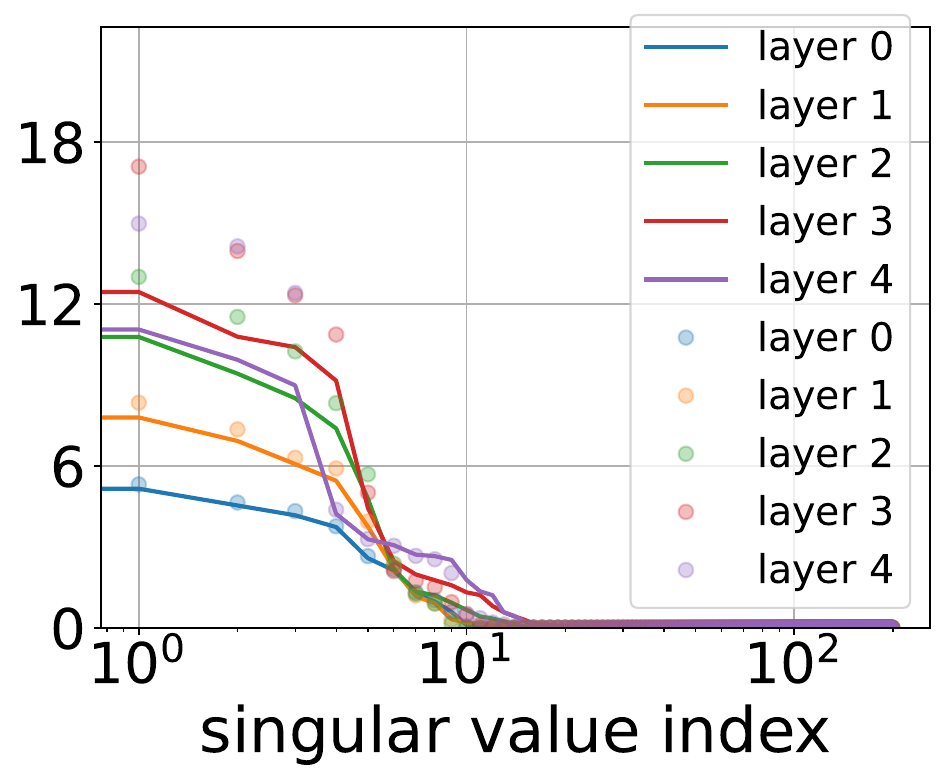} &   \includegraphics[width=0.2\textwidth]{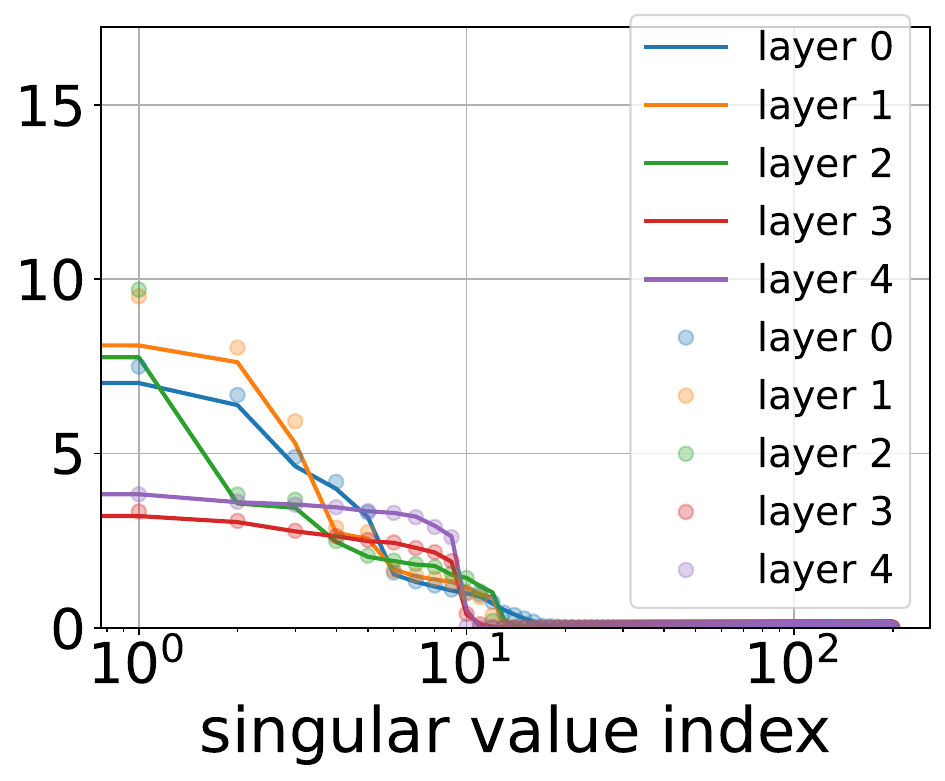} &
\includegraphics[width=0.2\textwidth]{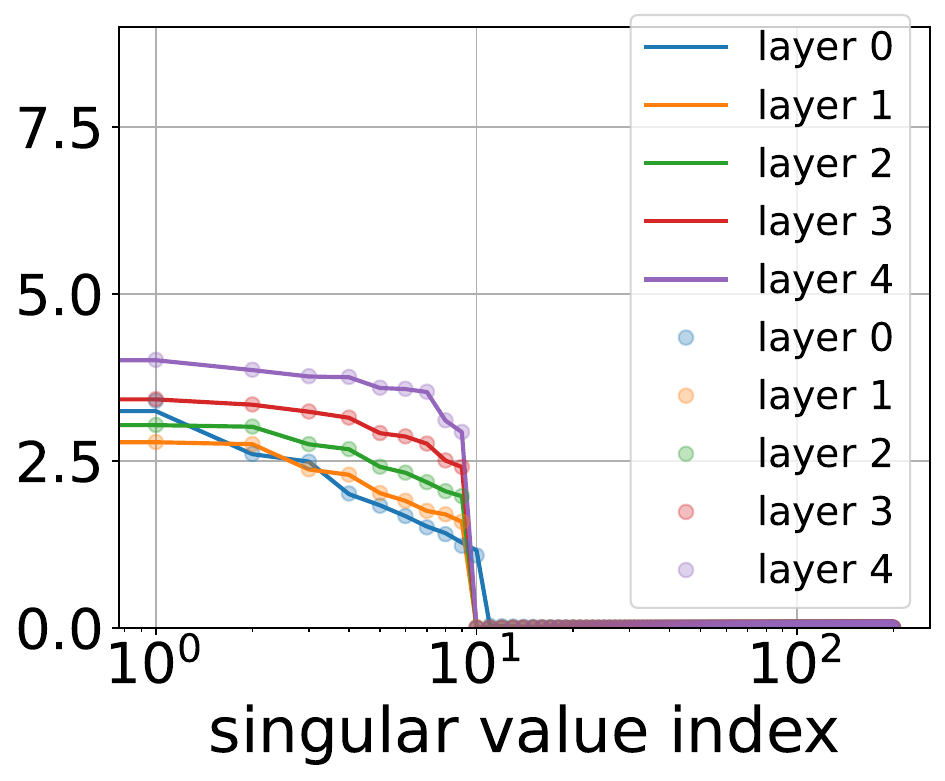} \\
$\delta = 10^{-1}$ & $\delta = 10^{-2}$ & $\delta = 0$ \\   
\includegraphics[width=0.2\textwidth]{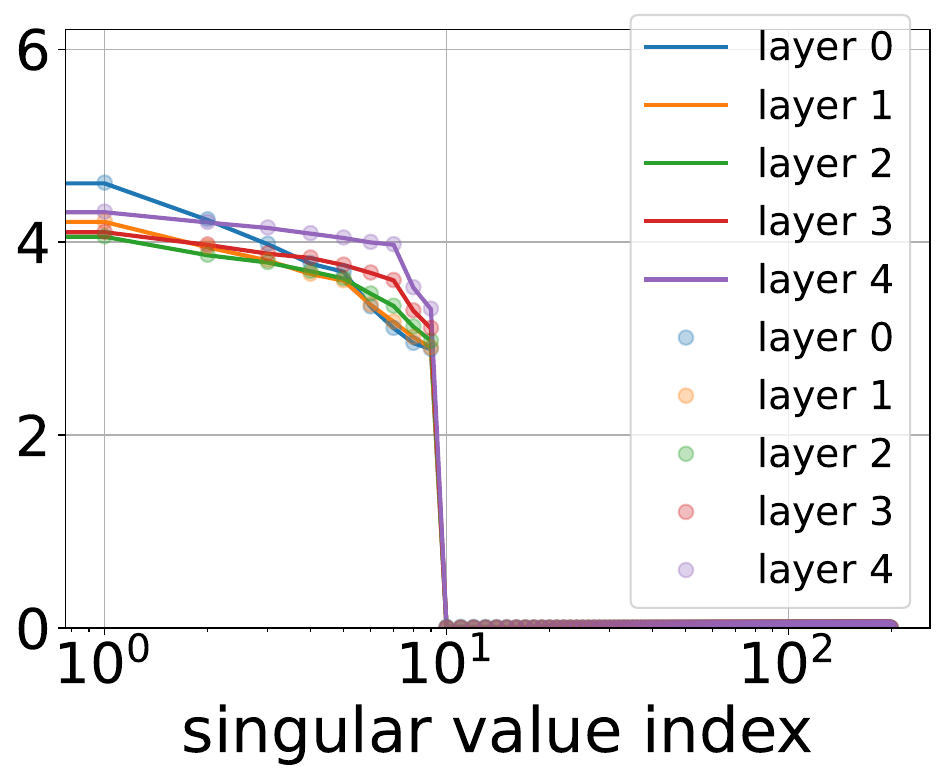} &
\includegraphics[width=0.2\textwidth]{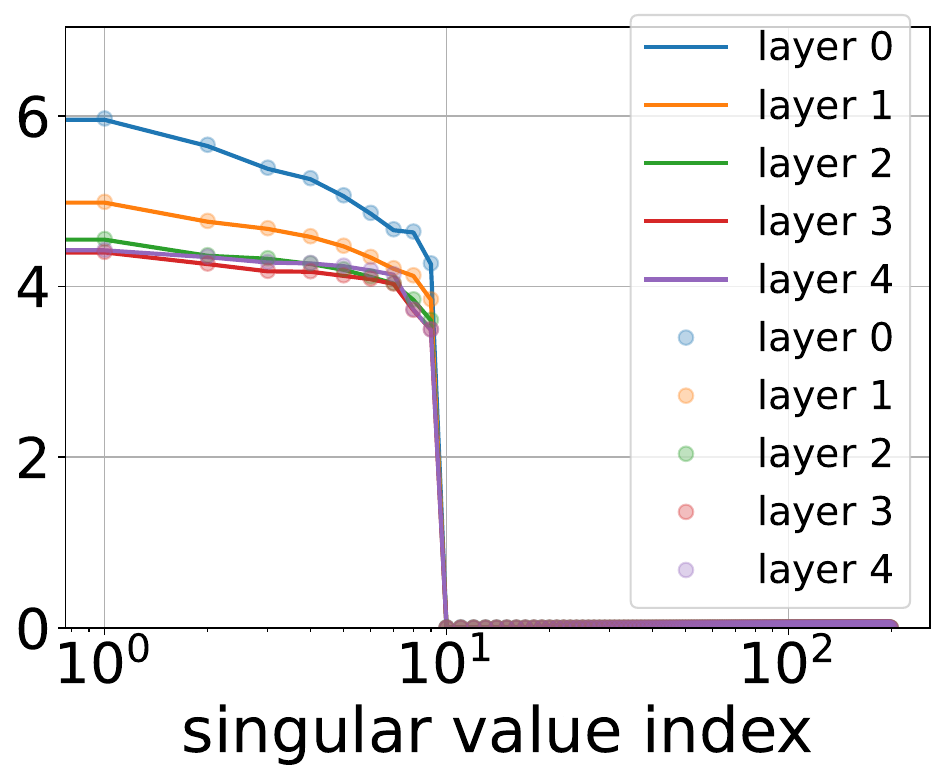} &   \includegraphics[width=0.2\textwidth]{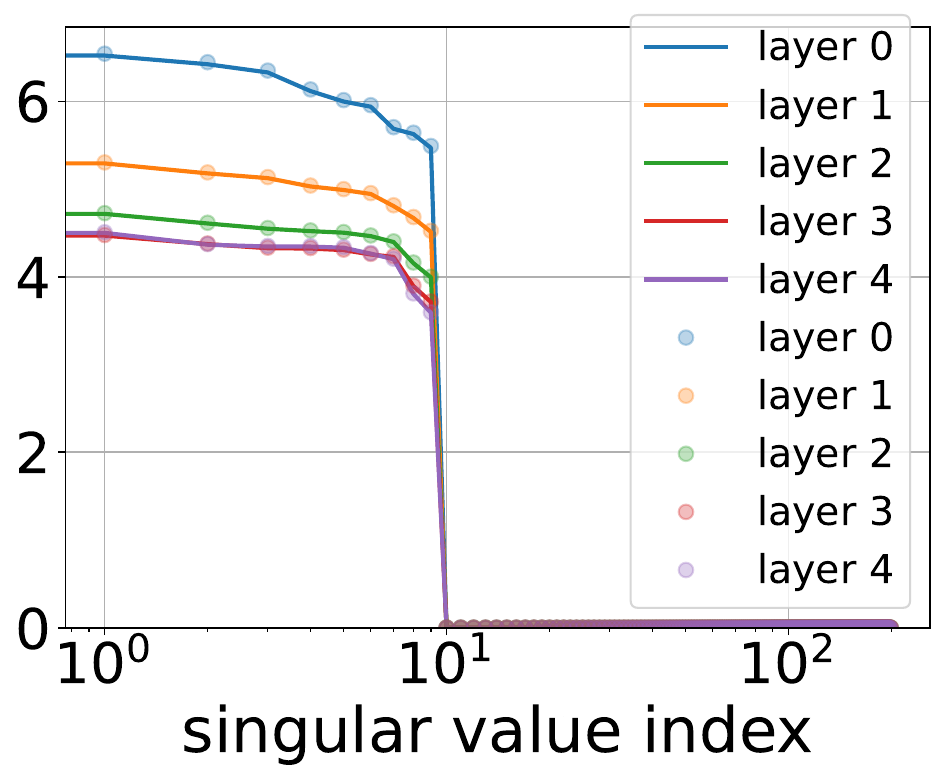} \\
\end{tabular}
\caption{Singular values of trained networks (full line) against those predicted by \Cref{thm:main_DNC_optimal_ff} (dots) for different values of $\delta = \frac{1}{\lambda} \in \{10,2,1,0.1,0.01,0\}$. The singular values are converging to the predicted ones as $\delta \to 0$, as shown by \Cref{prop:main_gammaconv_firstlayer} in combination with \Cref{thm:main_DNC_optimal_ff}.}
\label{fig:main_svals_prediction}
\end{figure}   

\subsection{Convergence to $\DNC$ Configurations}
In this experiment, we will showcase the convergence of $\TCV$ to zero of all intermediate layers in the setting of \Cref{thm:main_DNC_optimal_ff}. The experimental setting is the same as in \Cref{subsec:convergence_to_svals}, but with $L = 10$ layers and $\delta = 0$ fixed. In \Cref{fig:tcv_and_sranks_averaged} (rightmost figure), we plot for a single run the $\TCV$ of all intermediate post-activation embeddings. As we can observe from this first plot, the $\TCV$ of all intermediate embeddings goes essentially to machine precision as expected from \Cref{thm:main_DNC_optimal_ff,thm:main_reachability_nc}.
In \Cref{fig:tcv_and_sranks_averaged} (first three figures from the left), we instead plot the average $\TCV$ across all intermediate layers, for $100$ different random unitary initializations, to avoid slow convergence caused by the depth. As we can see, the $\TCV$ consistently goes to zero in all runs. Regarding the numerical ranks, we observe that they are all close to the optimal ones ($K = 10$).
%
\begin{figure}[t]
\centering
\resizebox{0.8\textwidth}{!}{
\bgroup
\setlength{\tabcolsep}{0.0mm}
\begin{tabular}{ccccc}
    \includegraphics[width=0.32\textwidth]{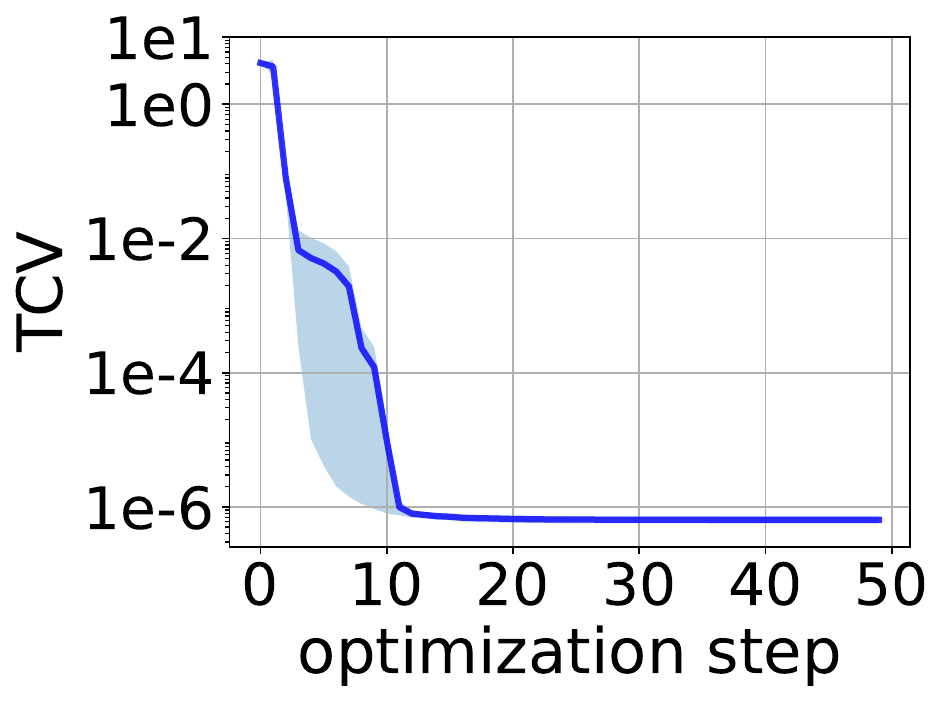}  &
     \includegraphics[width=0.315\textwidth]{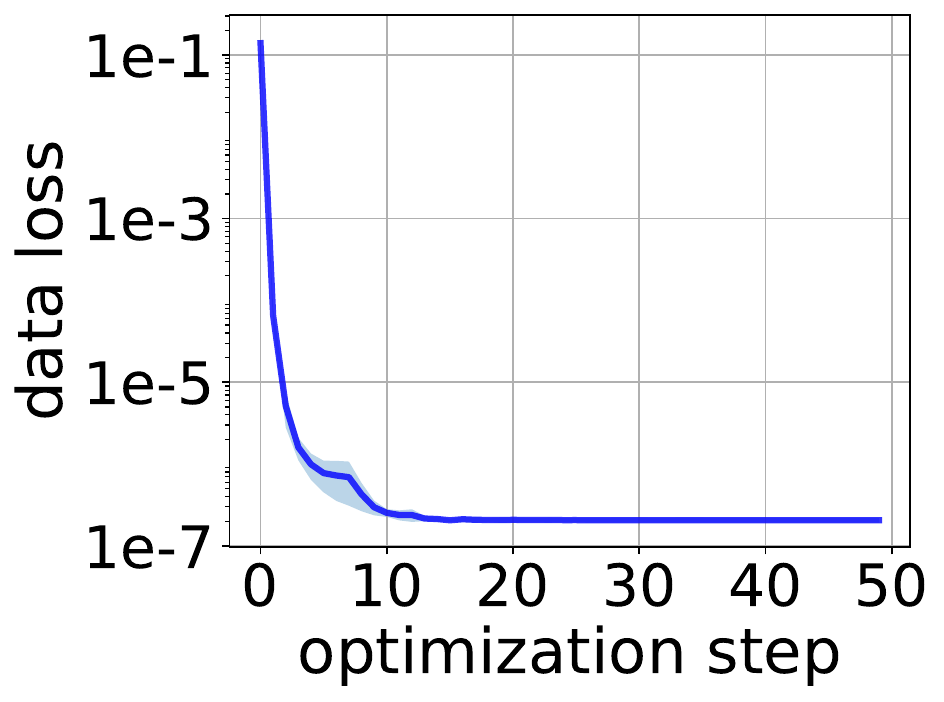}  &
     \includegraphics[width=0.32\textwidth]{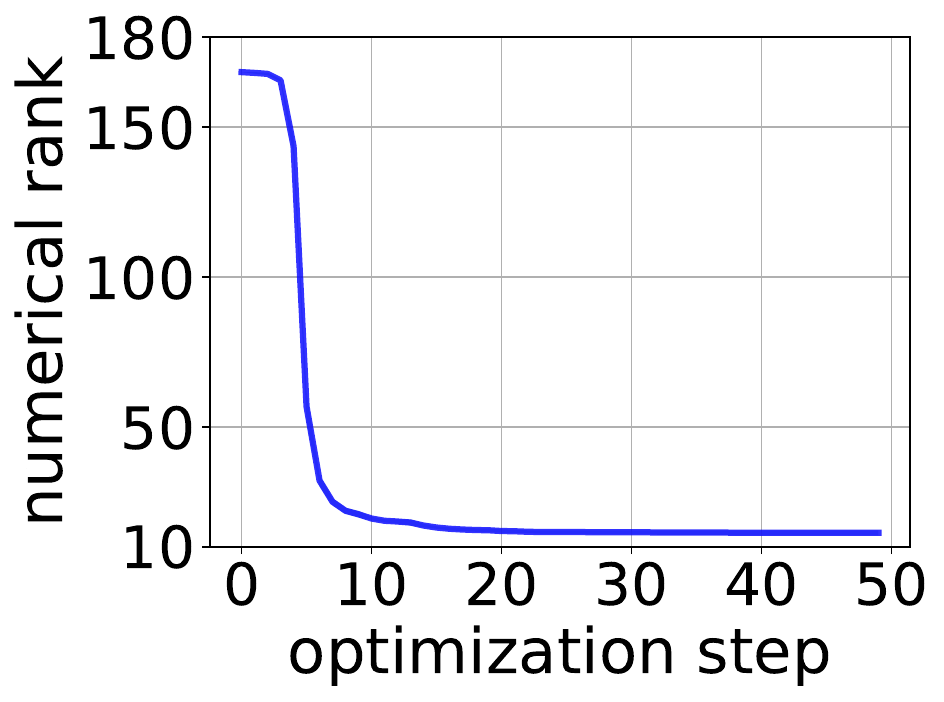} &

      $\qquad$
      &
      \includegraphics[width=0.32\textwidth]{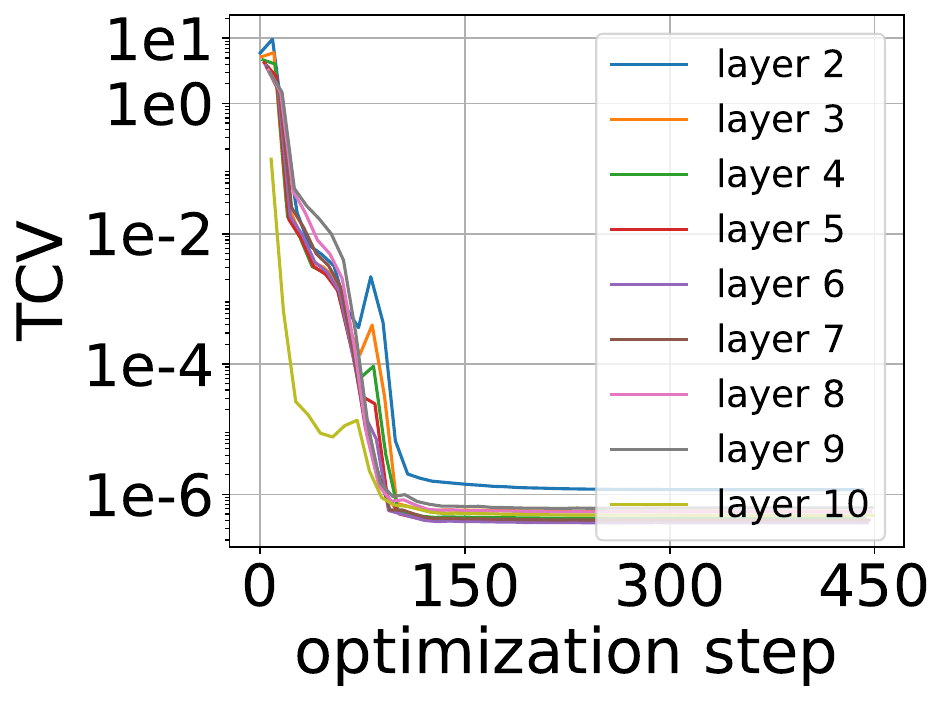}
\end{tabular}
\egroup
}
\caption{First three plots from the left: average of different quantities on all intermediate layers and for $100$ different random initializations. Shadows represent the area included between the minimal and maximal values observed. Each optimization step represents 100 epochs. Right: TCV of every layer during training.}
\label{fig:tcv_and_sranks_averaged}
\end{figure}

\section{Conclusions}
We proposed a theoretical investigation on the emergence of low-rank weight matrices and deep-neural collapse for a general class of neural networks. In particular, we established a quantitative relationship between deep neural collapse and the rank collapse of the weight matrices. In addition, we proved that for a general class of models, by not regularizing the first layer, $\DNC$ is optimal for all intermediate representations under a set of natural constraints, and that essentially there is no loss barrier between the interpolating global minima and $\DNC$ configurations. Lastly, we showed that the $\DNC$ configurations are stable and are preserved even when the first weight decay is included. Future investigations will go in the direction of relaxing some constraints and trying to adapt this approach to attention-based networks and GNNs. Moreover, we expect the same tools to be of interest when studying the representation learning of DNNs in general, which is also planned for future investigation.

\FloatBarrier


\bibliography{references.bib}
\bibliographystyle{tmlr}



\clearpage
\appendix
\section{Useful Definitions and Results}\label{sec:useful_facts}

In this section, we will introduce some useful definitions and fact which we will need to prove the results in \Cref{sec:ff_global}.
We start recalling the definition and some basic facts about $\Gamma$-convergence of functionals (see \citep{dalmaso1993introduction}).

\begin{definition}($\Gamma$-convergence)
\label{def:Gamma_convergence}
    Let $S$ be a locally numerable topological space and let $F_n: S \to \mathbb R$ be a sequence of functionals and $F: S \to \mathbb R$ be another functional. $F_n$ is said to $\Gamma$-converge to $F$ if the following hold:
    \begin{itemize}
        \item [(i)] For any sequence $(x_n) \subset S$ with $x_n \to x$, we have
        \[
        F(x) \leq \liminf_{n \to \infty} F_n(x_n).
        \]
        \item [(ii)] (Recovery sequence condition). For any $x \in S$, there exists a sequence $x_n \to x$ such that
        \[
        F(x) \geq \limsup_{n \to \infty} F_n(x_n).
        \]
    \end{itemize}
\end{definition}
To build some intuition, the first condition means that asymptotically, $F$ is a common lower bound for the sequence of $F_n$, at least pointwise. The second condition is a condition of recovery of the lower bound, implying, in a certain sense, optimality of the lower bound.
This definition can be useful when it is simpler to characterize the minima of $F_n$ than those of $F$, or viceversa. In our case, we can characterize the minima of $F$, but not those of $F_n$. However, thanks to $\Gamma$-convergence, we can establish that the minima of $F_n$ are going be close to those of $F$ in suitable scenarios. In particular, the following well known result holds:
\begin{proposition}\label{prop:convergence_of_minimizers}\citep{dalmaso1993introduction}
 Let $F_n \to F$ in the sense of $\Gamma$-convergence and let $S$ be sequentially compact. Then, for every sequence $x_n \in \arg\min_{y \in S} F_n(y)$ and every limit point $x$ of $(x_n)$, we have
 \[
 x \in \arg\min_{y} F(y).
 \]
\end{proposition}

The next lemma discusses the zeros of single-variable analytic functions, showing that their cardinality is always finite on compact sets, which we will need later. The results is well known and can be found in \citep{kranz_analytic_functions}.

\begin{lemma}[Finite number of zero of real analytic functions on compact intervals]\label{lemma:zero_analytic_functions}
Consider a nonzero real analytic function $f:[0,1] \to \R$. Then $f$ has a finite number of zeros in $[0,1]$.
\end{lemma}
\begin{proof}
First, we show that zeros are isolated. Let $t_0 \in [0,1]$ be a zero of $f$. Since $f$ is real analytic, then we can write
\[
f(t) = \sum_{k=0}^{+\infty} a_k (t-t_0)^k, \quad |t-t_0| \leq \rho
\]
since $f\ne 0$, there exists $m$ such that $a_m \ne 0$. Therefore,
\[
f(t) = (t-t_0)^m h(t), \quad h(t):=\sum_{k=0}^{+\infty} a_{m+k}(t-t_0)^k, \quad h(t_0) = a_m \ne 0 
\]
Therefore, by continuity of $h$ there exists a neighborhood of $t_0$ of radius $\rho'$ such that $h(t) \ne 0$ for all $|t-t_0| < \rho'$.
Then in this interval we have $f(t) = 0 \iff t = t_0$, proving that $t_0$ is isolated. Now, assume $f$ has infinitely many zeros $(t_k)$ in $[0,1]$, then by compactness there exists a convergent subsequence $t_{k_j} \to \bar t \in [0,1]$. But then, $\bar t$ would be an accumulation point, and this is only possible if $f$ is zero in an interval around $\bar t$. Since $f$ is analytic, this would force $f \equiv 0$, which is a contradiction.
\end{proof}

\begin{proposition}[Lower semicontinuity of functions defined through infimum]\label{prop:semicont}
    Let $X,Y$ be two finite-dimensional normed spaces, $f: X \times Y \to \mathbb R$ and $g: X \times Y \to \mathbb R^m$ be continuous functions, and $\Omega(x) = \{ y \in Y \,| g(x,y) = 0\, \}$. Moreover, assume that $f(x,y)$ is coercive with respect to $y$, uniformly in $x$, i.e. that there exists $C>0$ such that $\|f(x,y)\|\geq C\|y\|$ for all $x\in X$. Then, the function
    \[
    F(x) = \inf_{y \in \Omega(x)} f(x,y)
    \]
    is lower semicontinuous.
\end{proposition}
\begin{proof}
    To prove the lower semicontinuity of $F(x)$ we need to show that, for every sequence $x_k \to \bar x$, we have
    \[
    \liminf_k F(x_k) \geq F(\bar{x}).
    \]
    Let $(x_{k_j})$ be a subsequence of $(x_k)$ such that $\exists\lim_j F(x_{k_j})=\liminf_k F(x_k)$
    Then by definition of infimum, for every $x_{k}$ we can find $y_k \in \Omega(x_k)$ such that $F(x_k)\leq f(x_k,y_k) \leq F(x_k) + \frac{1}{k}$. Then we have two cases:
    \begin{itemize}
        \item [(i)] If the sequence $(y_{k_j})$ is unbounded, up to considering a subsequence, we can assume that $\| y_{k_j} \| \to \infty$. Then, for the  uniform coercivity of $f$ we have
        \[
        f(x_{k_j},y_{k_j}) \to \infty
        \]
        and, therefore,
        \[
        \liminf_k F(x_{k})=\lim_j F(x_{k_j}) \geq \lim_j \left(f(x_{k_j},y_{k_j})-\frac{1}{{k_j}}\right) = + \infty \geq F(\bar x).
        \]
        
        \item [(ii)] If the sequence $(y_{k_j})$ 
        is bounded, up to considering a subsequence, we can assume $y_{k_j} \to \bar y$. Since $y_{k_j} \in \Omega(x_{k_j})$ for all $j$, by the continuity of $g$, $0 \equiv g(x_{k_j},y_{k_j}) \to g(\bar x, \bar y)$. Therefore, $ \bar y \in \Omega(\bar x)$. Similarly, the continuity of $f$ gives
        \[
        f(x_{k_j},y_{k_j}) \to f(\bar x,\bar y).
        \]
        Therefore, we see that
        \[
        \liminf_k F(x_k) = \lim_j F(x_{k_j})\geq \lim_j \left(f(x_{k_j},y_{k_j})-\frac{1}{k_j}\right) = f(\bar x,\bar y) \geq \inf_{y \in \Omega(\bar x)} f(\bar x,y) = F(\bar x),
        \]
        which is the required inequality.
    \end{itemize}
    This concludes the proof.
\end{proof}

\begin{remark}\label{remark:lsc_of_f}
    We will use this last result in \Cref{subsec:ff_reachability} to prove the lower semicontinuity of $F(A,B) = \|AB^+\|^2$. In fact, by defining 
    \[x = (A,B) \in X = \{ (A,B) \in \mathbb R^{n \times N} \times \mathbb R^{m \times N} \, | \, AB^+B = A \} \quad \text{ and }\quad  y = W \in \mathbb R^{n \times m}\] 
    we can define $g(x,y) = WB - A$,  $f(x,y) = \| W \|^2$. Then, recalling that $\Omega(x)=\{y=W\;|\; g(x,y)=WB-A=0\}$, a standard result, see e.g. \citep[Section~5.5.1]{golub2013matrix}, yields 
\[
F(x) = \inf_{y \in \Omega(x)} f(x,y).
\]

In particular we are exactly in the setting of \Cref{prop:semicont}, proving lower semicontinuity of $F(A,B) = \|A B^+\|^2$. 
\end{remark}

\begin{lemma}[Blowup of $\LimitLoss$ outside the natural constraint]\label{lemma:blowup}
 Consider a fixed matrix $A \in \mathbb R^{n \times N}$ with $N \geq n \geq \rank(A) = r$ and the matrix functions $f, g: \mathbb R^{m \times N} \to \mathbb R$ defined as $g(B) = \|AB^+B-A\|^2$ and $f(B) = \|AB^+ \|$ with $m \geq n$.
    Let $B_k \rightarrow \bar{B}$ with $g(B_k) = 0$ for every $k$ and $g(\bar B) \ne 0$. Then, $\lim_{k}f(B_k)=+\infty$. 
\end{lemma}
\begin{proof}
    To prove it, let $U_k\Sigma_k V_k^T$ be the reduced SVD decomposition of $B_k$
and $\bar{U}\bar{\Sigma} \bar{V}^T$ the one of $\bar{B}$. Then by assumption we have $A=A \bar{V}\bar{V}^T+ A(I-\bar{V}\bar{V}^T)$ with the second term on the right hand side being different from zero because $g(\bar B) \ne 0$. Additionally, w.l.o.g., for any $k$ we can write $V_k=[V_{k,1}, V_{k,2}]$, $\Sigma_k=\mathrm{diag}[\Sigma_{k,1}, \Sigma_{k,2}]$ and $U_k=[U_{k,1}, U_{k,2}]$ where $[U_{k,1}, \Sigma_{k,1}, V_{k,1}]\rightarrow [\bar{U}, \bar{\Sigma}, \bar{V}]$,  $I-V_{k,1}{V_{k,1}}^T\rightarrow(I-\bar{V}\bar{V}^T)$ and $\Sigma_{k,2}\rightarrow 0$. Then 
\begin{equation*}
\begin{aligned}
    f(B_k)&=\|A V_k^1\Sigma_{k,1}^{-1} U_{k,1}^T+ A V_{k,2}\Sigma_{k,2}^{-1} U_{k,2}^T\|\geq \|A V_{k,2}\Sigma_{k,2}^{-1} \|-\|A V_{k,1}\Sigma_{k,1}^{-1} \|  \\
    &\geq \min_i{(\Sigma_{k,2})_{ii}^{-1}}\|A V_{k,2}\|-\|A V_{k,1}\Sigma_{k,1}^{-1} \|\rightarrow +\infty-\|A\bar{V}\bar{\Sigma}^{-1}\|=+\infty,
\end{aligned}
\end{equation*}
where we have used that the singular values are all nonnegative, i.e., $\Sigma_{k_{ii}}\geq 0$ for all $i$ and $k$ with ${\Sigma_{k,2}}_{ii}\rightarrow 0$ for any i, that $\lim_{k} AV_{k,2}=A(I-\bar{V})\neq 0$ and the equalities $\|A V_{k,1}\Sigma_{k,1}^{-1} U_{k,1}^T\|^2= Tr(A V_{k,1} \Sigma_{k,1}^{-1} U_{k,1}^T U_{k,1}\Sigma_{k,1}^{-1}V_{k,1}^T A)=Tr(A V_{k,1}\Sigma_{k,1}^{-1} \Sigma_{k,1}^{-1}V_{k,1}^T A)=\|A V_{k,1}\Sigma_{k,1}^{-1}\|^2$ and analogously $\|A V_{k,2}\Sigma_{k,2}^{-1} U_{k,2}^T\|=\|A V_{k,2}\Sigma_{k,2}^{-1}\|$.
\end{proof}

\subsection{Proof of Proposition \ref{Prop_rank_of_lossgradient_at_single_data}}
\label{Proof_Prop_rank_of_lossgradient_at_single_data}
    The main step of the proof is to show that the gradient with respect to a single data point produces matrices of rank at most 1. One then concludes by observing that the sum of $K$ terms of rank at most $1$ has rank at most $K$.
    Given the structure of the loss function that depends only on the output of the model $f(\param,x)$,
    if we use the chain rule and the recursive definition of $f(\param,x)$ in \eqref{DEF_Neural_Network}, we obtain
    \begin{align*}
    \nabla_{W_{l}} \loss\big(f(\Theta,x),y\big) &= \nabla_{W_{l}} \loss\big(f_{l}(\param_{l},z^{l}(x)),y\big) 
    = \frac{\partial f(\Theta,x)}{\partial W_l}^\top \nabla_f \loss\big(f(\Theta,x),y\big) 
    \\& = \frac{\partial z^l(x)}{\partial W_l}^\top \frac{\partial f_{l}(\param_{l},z^{l}(x))}{\partial z^l}^\top \nabla_f \loss\big(f(\Theta,x),y\big).  
    \end{align*}
    
    Now we use the assumption that $z^l = \sigma(W_l z^{l-1}+b_l) = \sigma(a^l)$,  where we have introduced the auxiliary variable $a^l := W_lz^{l-1}+b_l$. This finally leads to 
\begin{align*}
    \nabla_{W_{l}} \loss\big(f(\param,x),y\big) 
    & = \big(z^{l-1}(x) \otimes I\big) \frac{\partial z^L(x)}{\partial a^l}^\top \nabla_{z^L} \loss\big(z^L(x),y\big) \\
    & = 
    \frac{\partial f_{l}(\param_{l},\sigma(a^{l}(x)))}{\partial a^l}^\top \Bigl(\nabla_f \loss\big(f(\Theta,x),y\big) \cdot z^{l-1}(x)^\top  \Bigr).
\end{align*}
    By noticing that the last term $\nabla_{f} \loss\big(f(\Theta,x),y\big) z^l(x)^\top$ is a matrix of rank at most 1, we can conclude by submultiplicativity of the rank. \hfill $\square$

\subsection{Proof of Lemma \ref{Lemma_approx__fullgradient_by_centroid_based_loss_gradient}}\label{Lemma3.3Proof}
For simplicity, we will omit the superscript $j$ from $z^j_i$.
We start by considering the Taylor expansion of the gradient $\nabla_{\layer_l}\loss\big(f_j(\param_{j},z),y\big)$ centered in  $\big(\bar{z}_k,\bar{y}_k\big)$ for any point $\big(z_i,y_i\big)$ with $i\in \mathcal{C}_k$ and any $k=1,\dots,K$
\begin{equation*}
\begin{aligned}
\grad_{\layer_l} \loss\big(f_j(\param_{j}, z_i), y_i \big) 
=& \grad_{\layer_l} \loss\big(f_j(\param_{j}, \bar{z}_k), \bar{y}_k\big)\\
&+ \grad_{(z,y)}\nabla_{\layer_l} \loss\big(f_j(\param_{j}, \bar{z}_k), \bar{y}_k\big)\big[(z_i, y_i)-(\bar{z}_k, \bar{y}_k)\big] + R_i^k(\param),
\end{aligned}
\end{equation*}
where the reminder is given by
\begin{equation*}
R_i^k(\param) := \frac12\grad_{(z,y)}^2\nabla_{\layer_l} \loss\big(f_j(\param_{j}, \alpha_i^k) \beta_i^k\big)[(z_i, y_i)-(\bar{z}_k, \bar{y}_k), (z_i, y_i)-(\bar{z}_k, \bar{y}_k)],    
\end{equation*}
for some $(\alpha_i^k, \beta_i^k)$ between $(z_i, y_i)$ and $(\bar{z}_k, \bar{y}_k)$.
Therefore, using the linearity of $\grad_{(z,y)}\grad_{\layer_l} \loss[\cdot]$, we have
\begin{align*}
\grad_{\layer_l}\Loss_0^j(\param) 
 = &\frac{1}{N} \sum_{j=1}^N \nabla_{\layer_l}\loss\big(f_j(\param_{j}, z_j), y_j\big)=\\
 = &\frac{1}{N} \sum_{k=1}^K \sum_{i\in \mathcal{C}_k} \grad_{\layer_l}\loss\big(f_j(\param_{j}, z_i), y_i\big)=\\
 = & \frac{1}{N} \sum_{k=1}^K \sum_{i\in \mathcal{C}_k} \grad_{\layer_l} \loss\big(f_j(\param_{j}, \bar{z}_k), \bar{y}_k\big)+ \\&
 + \grad_{(z,y)}\grad_{\layer_l} \loss\big(f_j(\param_{j}, \bar{z}_k), \bar{y}_k\big)[(z_i, y_i) - (\bar{z}_k, \bar{y}_k)] + R_i^k(\param)=\\
 = & \sum_{k=1}^K \pi_k \grad_{\layer_l} \loss\big(f_j(\param_{j}, \bar{z}_k), \bar{y}_k\big) +\\
 &+ \frac{1}{N}\sum_{k=1}^K\grad_{(z,y)}\grad_{\layer_l} \loss\big(f_j(\param_{j}, \bar{z}_k), \bar{y}_k\big)\bigg[\underbrace{\sum_{i\in \mathcal{C}_k}((z_i, y_i) - (\bar{z}_k, \bar{y}_k))}_{=0}\bigg] +\\ 
& + \frac{1}{N} \sum_{k=1}^K \sum_{i\in \mathcal{C}_k} R_i^k(\param) =: \grad_{\layer_l} \Loss_0^{j,\mathcal{C}}(\param) + R(\param),
\end{align*}
where $\pi_k=N_k/N$.
%
Hence, we obtain  
\begin{align*}
 \|R(\param)\|  \leq & \frac{1}{N} \sum_{k=1}^K \sum_{i\in \mathcal{C}_k} \|R_i^k(\param)\|  \leq \frac{M_{\loss,j}(\param)}{2N} \sum_{k=1}^K \sum_{i\in \mathcal{C}_k}\|(z_i^k, y_i^k) - (\bar{z}_k, \bar{y}_k)\|^2= \frac{M_{\loss,j}(\param)}{2} \WCSS_j,
\end{align*}
where 
\begin{align*}
    M_{\loss,j}(\param)
    &=\sup_{k=1,\dots,K}\sup_{i\in \mathcal{C}_k}\|\grad^2_{(z,y)}\grad_{\layer_l}\loss\big(f_j(\param_{j},\alpha_i^k),\beta_i^k\big)\| \leq
     \\ &\leq \sup_{k=1,\dots,K}\sup_{(\alpha,\beta)\in \mathcal{H}_k^j}\|\grad^2_{(z,y)}\grad_{\layer_l}\loss\big(f_j(\param_{j},\alpha),\beta\big)\|,
\end{align*}
as desired. \hfill $\square$

\subsection{Proof of \Cref{Theorem_best_rank_k-approx_bound}}
\label{Theorem3.4Proof}
    Observe that from \eqref{Eq_critical_point_eq_regularized_loss}, at the stationary point we have 
    \begin{equation}\label{Thmeq_critical_point_eq}
        \layer^*_l=-\frac{1}{ \lambda}\grad_{\layer_l}\Loss_0(\param^*).
    \end{equation}
    Now, Lemma \ref{Lemma_approx__fullgradient_by_centroid_based_loss_gradient} allows us to control the distance of the gradient in the right-hand side of \eqref{Thmeq_critical_point_eq} from $\matrices_K^{d_{l-1},d_l}$. In particular, given $\mathcal{C}\in \cup_{r=1}^K\mathcal{P}_r$, for any $j=1,\dots,l-1$ we have
    \begin{equation*}
    \begin{aligned}
    \min_{Z\in \cup_{r=1}^K\matrices_r} \|\layer^*_l - Z\|
    &\leq \frac{\|\grad_{W_l} \Loss_0(\param^*)- \grad_{W_l} \Loss_0^{j,\mathcal{C}}(\param^*)\|}{\lambda}\\
    &\leq\frac{M_{l,j}(\param^*,\mathcal{C})\WCSS_{j}(\mathcal{C})}{\lambda},
    \end{aligned}
    \end{equation*}
    where we have used that $\grad_{W_l} \Loss_0^{j,\mathcal{C}}(\param^*)/\lambda$ has rank smaller than or equal to the number of families of $\mathcal{C}$.
    Then, first minimizing over $j=1,\dots,l-1$, second bounding from above $M_{l,j}(\param^*,\mathcal{C})\leq M_{l,j}(\param^*)$ for any $j$, and third minimizing over all the possible partitions $\mathcal{C}\in  \cup_{r=1}^K\mathcal{P}_r$ yield the thesis.

\section{Global Minima of Representation Cost for Feedforward Networks}\label{sec:ff_global}

In this section, we will prove the results presented in \Cref{sec:main_DNC} together with the necessary lemmas. As we already discussed in \Cref{sec:main_DNC}, the main idea is to study the minima of $\LimitLoss$ recursively layer by layer, i.e., to study the minima of the local terms $\LimitLoss_{loc}(Z^{l+1},Z^l) = \frac12\|\sigma_{l+1}^{-1}(Z^{l+1})Z^{l,+} \|^2$ in $Z^l$ for $Z^{l+1}$ fixed. We characterize such minimum, under the constraints considered in \Cref{thm:main_DNC_optimal_ff}, in the next technical lemma. This will serve as a basis for recursively proving the main \Cref{thm:main_DNC_optimal_ff}.
We start formally introducing  the constraints. Consider $A \in \mathbb{R}^{n \times N}$ a fixed matrix playing the role of $\sigma_{l+1}^{-1}(Z^{l+1})$, with $\rank(A) = r$, and $\{s_i(A)\}$ its singular values. Then, for some fixed $m\geq n$ and $C>0$ we write 
\begin{equation}\label{eq:local_constraint}
    \mathcal S_{loc}(A):=\Bigl\{B \in\mathbb{R}^{m \times N} : s_i(B)^2 \leq Cs_i(A)^2\,\, \forall i = 1,\dots,r,\,\, AB^+B = A \Bigr\}.
\end{equation}

\begin{lemma}[Optimal single layer for feedforward representation cost]\label{lemma:ff_lemma_svals}
Let $A \in \mathbb{R}^{n \times N}$ be a fixed matrix with $\rank(A) = r$, and $\{s_i(A)\}$ its singular values and let also $\localF(B) = \frac{1}{2}\| AB^+\|^2$.
    Then, for $n \leq m$ and any fixed constant $C>0$  in the definition of $\mathcal{S}_{loc}(A)$, we have that the solution to the  constrained optimization problem $\min_{B \in \mathcal S_{loc}(A)}\localF(B)$  is attained at all points
    \[
B^* = OA, \quad O \in \mathbb R^{m \times n}, \quad O^\top O = I.
    \]
   In particular, for any optimal $B^*$, we have 
    $\LimitLoss(B^*)=\rank(A)/2C$. Moreover, for $n=m = N$, and for Lebesgue-almost every $B_0\in \mathcal S$, there exists a continuous path $\gamma:[0,1] \to \R^{N \times N}$ with $\gamma(0) = B_0$, $\gamma(1) \in \underset{B \in \mathcal S_{loc}(A)}{\arg\min} \, \localF(B)$ and $\frac{d}{dt}\localF(\gamma(t)) \leq 0$.

\end{lemma}
\begin{proof}
First of all, we prove the existence of a minimizer. Observe that whenever a sequence $B_k \to \bar B$ with $A\bar B^+ \bar B \ne A$, then $\localF(B_k) \to +\infty$ thanks to \Cref{lemma:blowup}. Therefore, we can restrict our attention to sequences $B_k \to \bar B$ with $A \bar B^+ \bar B = A$ (and therefore in particular with $\rank(\bar B) \geq \rank(A)$).

Since the constraint set $\mathcal S_{loc}(A)$ is bounded, any sequence $B_k$ for which $\localF(B_k) \to \inf_{B \in \mathcal S} \localF(B)$ is also bounded.
Therefore, there exists a convergent subsequence $B_{k_j} \to \bar B$ with $\rank(\bar B) \geq \rank(A)$. By lower semicontinuity of $\localF$ (shown in \Cref{remark:lsc_of_f}) we obtain
\[
\lim_j \localF(B_{k_j}) = \liminf_j \localF(B_{k_j}) = \inf_{B \in \mathcal S} \localF(B)  \geq \localF(\bar B).
\]
Since the singular values $s_i(B)$ are continuous functions of $B$, the inequality $s_i(B)^2 \leq C s_i(A)^2$ gives a closed set, and therefore $s_i(\bar B)^2 \leq C s_i( A)^2$ .
Moreover, we restricted without loss of generality to sequences $(B_k)$ for which $A\bar B^+\bar B = A$ and therefore we have $\bar B \in S_{loc}$. Combining $\bar B \in S_{loc}$ with $\inf_{B \in \mathcal S} \localF(B)  \geq \localF(\bar B)$ we get that $\bar B$ is a minimizer of $\localF$ in $\mathcal{S}_{loc}(A)$.

    To understand the structure of minimizers, decompose $A,B$ using a reduced SVD to get $A = U_A \Sigma_A V_A^\top, B = U_B \Sigma_BV_B^\top$. Notice that the constraint $AB^+ B = A$ implies that $\mathrm{rank}(A) \leq \mathrm{rank}(B)$.  Then the constraint is satisfied if and only if
    \[
    U_A \Sigma_A V_A^\top V_B \Sigma_B^+ U_B^\top U_B \Sigma_B V_B^\top = U_A \Sigma_A V_A^\top V_B V_B^\top = U_A \Sigma_A V_A^\top.
    \]
    Multiplying on the left by $\Sigma_A^+ U_A^\top$
    we obtain
    \[
    V_A^\top V_B V_B^\top = V_A^\top
    \]
    and by multiplying on the right by $V_A$ we get
    \[
    \underbrace{(V_A^\top V_B)}_{=:Q} (V_B^\top V_A) = Q Q^\top = I.
    \]
    Concerning the function value, by substituting $V_A^\top V_B = Q$ we see that
    \begin{equation}\label{eq:local_weightdecay}
    \begin{split}
    \localF( U_B \Sigma_B V_B) &= \frac{1}{2} \| U_A \Sigma_A V_A^\top  V_B\Sigma_B^+ U_B^\top\|^2 = \frac{1}{2} \| U_A \Sigma_A Q \Sigma_B^+ U_B^\top\|^2 =\\
    & =\frac{1}{2} \mathrm{tr}(U_A \Sigma_A Q \Sigma_B^+ U_B^\top U_B (\Sigma_B^+)^\top Q^\top \Sigma_A^\top U_A^\top) =\\
    &=\frac{1}{2}\mathrm{tr}(\Sigma_A Q \Sigma_B^+ (\Sigma_B^+)^\top Q^\top \Sigma_A^\top ) = \frac{1}{2 } \|\Sigma_A Q \Sigma_B^+ \|^2 =\\
    & = \frac{1}{2 } \sum_{i = 1}^r \sum_{j=1}^m s_i(A)^2 s_j(B)^{-2} Q_{ij}^2 = \frac{1}{2 } \sum_{j = 1}^m s_j(B)^{-2} \sum_{i=1}^r s_i(A)^2 Q_{ij}^2.
    \end{split}
    \end{equation}
    Assume without loss of generality that the singular values of both $A,B$ are ordered in a descending manner.
    Now, since $QQ^\top = I$, this is a classical form of Brockett's cost function \citep{BROCKETT199179}, whose solution is given by $Q^* = [I_r,0_{r,m-r}]$ by using Theorem 2.1 in \citep{Liang_2023}, which is a generalization of the classical Ky-Fan trace inequality.
    Then, we have
    \[
    \min_{s_j(B)>0} \frac12\sum_{j = 1}^r \frac{s_j(A)^2}{s_j(B)^2}, \quad \text{s.t.} \, s_j(B)^2 \leq Cs_j(A)^2, \forall j = 1,\dots,r.
    \]
    which is minimized by $s_k(B)^{2} =  Cs_k(A)^2$.
    The optimal function value is given by 
    \[
    \localF(U_B \Sigma_A^* V_A^\top) = \frac{1}{2} \sum_j \frac{s_j(A)^2}{Cs_j(A)^2} = \frac{1}{2C}\rank(A). 
    \]
    Therefore, we have that minima are attained with $V_B (Q^*)^\top = V_A$ for $Q^* = [I_r,0_{r,m-r}] \in \mathbb{R}^{r \times m}$, $\Sigma_B \propto \mathrm{blockdiag}(\Sigma_A,0_{m-r})$. This implies $V_B (Q^*)^\top Q^* = V_AQ^*$. Therefore, the $r$ dominant right eigenvectors of $B$ coincide with those of $A$. Since the last $m-r$ singular values of $B$ are zero, we can use the reduced SVD form and delete the last $m-r$ columns of $V_B$, therefore obtaining $V_B = V_A$ and $\Sigma_B \propto \Sigma_A$. Moreover, $U_B$ is free, therefore all the optimal points are
    \[
    B^* \propto OU_A \Sigma_AV_A^\top = OA, \quad  O \in \mathbb R^{m \times n}, O^\top O = I  \, (\text{since} \, n \leq m)
    \]
    and the optimal value is given by
    \[
    \localF(B^*) = \frac{\rank(A)^2}{2C}.
    \]

    Now, to prove the existence of the path $\gamma(t)$, we assume $n = m = N$ and, without loss of generality, we consider $C = 1$.
    From \Cref{eq:local_weightdecay} it is sufficient to find a path for the singular values $(s_i(t))$, and for the matrix $Q(t) \in O_N(\R)$ such that they make $\localF$ decrease and reach the optimal value.

    The existence of a descent path in $Q$, for fixed $\Sigma_A,\Sigma_B^+$, is guaranteed by \Cref{lemma:descent_Q}. Then, we can assume w.l.o.g. that $Q = I$ by assuming that the singular values of $A$ and $B$ are both ascendingly ordered.  
    So we concatenate the path in $Q$ with the following one for the singular values of $B$:
    \begin{equation}\label{eq:path_s}
    s_i(B)(t) = (1-t)s_i(B_0) + t s_i(A).
    \end{equation}
    The path in $Q$ decreases $\LimitLoss$ by \Cref{lemma:descent_Q}.
    Additionally, by taking the derivative in $t$ of $\localF$ along the path \eqref{eq:path_s}, from \eqref{eq:local_weightdecay} we have
    \begin{align*}
    \frac{d}{dt} \localF(OU_A \Sigma_B(t) Q(1)^\top V_A^\top) &= \frac{d}{dt}\sum_{i,j=1}^Ns_i^2(A) s_j^{-2}(B)(t) Q(1)^2_{ij} \\
    & = -2 \sum_{i,j=1}^Ns_i(A)^2 s_j(B)(t)^{-3}\delta_{ij}(s_j(A) - s_j(B)(0)) \\
    & =   -2 \sum_{i=1}^N s_i^2(A) s_i(B)(t)^{-3}(s_i(A) - s_i(B)(0))\leq 0,
    \end{align*}
    where in the second equality we have used that $Q(1) = I$ and the last inequality holds because $s_i(B_0) \leq s_i(A)$. 
     This concludes the proof.
\end{proof}
\begin{remark}[Local lemma for backbone]
    In the case of a deep-architecture, the second layer requires a dedicated analysis because of the double representation constraint with the first layer
    \[
    \min_{B \in \mathcal S_{loc}(A),\,\sigma_2^{-1}(B)Z^{1,+}Z^1 = \sigma_{2}^{-1}(B)}\LimitLoss_{loc}(B).
    \]
    In particular, if the columns of $Z^1$ are linearly independent, then the second representation constraint is trivially satisfied. In fact, in that case $Z^{1,+}Z^1 = Z^{1,-1}Z^1 = I_N$, reducing to the case presented in \Cref{lemma:ff_lemma_svals}. Thanks to this observation and \Cref{lemma:large_layer_lindip}, the following analysis holds in the pyramidal architecture hypothesis and with a fixed $W_1$ which satisfies the assumptions of \Cref{lemma:large_layer_lindip}.
\end{remark}

\begin{lemma}[Generic linear independence after a wide layer]\label{lemma:large_layer_lindip}
    Consider an analytic non-polynomial activation function $\sigma: \R \to \R$ and a set of data points $\{x_i \}_{i=1}^N\subset \R^{n}$ with $n \leq N \leq 2^{n-1}$. Then for, Lebesgue-almost every couple $(X,W) \in \R^{n \times N} \times \R^{N \times n}$ it holds that $\det(\sigma(WX)) \ne 0$, where $X = [x_1,\dots,x_N] \in \mathbb{R}^{n \times N}$ and $\sigma$ applies entrywise.
\end{lemma}
\begin{proof}
 Letting $Z(W,X) := WX$, consider the determinant    \[
    \det(\sigma(WX)) =: 
    F(Z(W,X)),
    \]
    which is analytic in $(W,X)$. Therefore, if we can prove that $F$ is not identically zero, then its zero set in $\R^{N \times n} \times \R^{n \times N}$ has Lebesgue measure zero \citep[Proposition 1]{Mityagin}, and we can conclude. To show this, we need to exhibit a single matrix $Z$ with $\mathrm{rank}(Z) \leq n$ for which $F(Z)\ne 0$. Following \citep[Theorem 2.5]{chu2024metric}, whenever $\sigma$ is non-polynomial, then $\sigma$ is not rank preserving ( see \citep[Definition 2.2]{chu2024metric}), meaning that there exists a matrix $Z$ with $\rank(Z) \leq \log_2(N)+1 \leq n$ such that $\rank(\sigma(Z)) = N$. This concludes the proof.
\end{proof}

\begin{lemma}[Descent curve on $O_N(\R)$]\label{lemma:descent_Q}
Let $D_1,D_2 \in \R^{N \times N}$ be fixed diagonal positive definite matrices, each of them with distinct eigenvalues, and define
\[
\phi: O_N(\R) \to \R, \quad \phi(Q):= \|D_1 Q D_2 \|^2 = \mathrm{tr}(D_1^2 Q D_2^2Q^\top).
\]
Then, for almost every $Q_0 \in O_N(\R)$ there exists an analytic curve $Q: [0,1] \to O_N(\R)$ such that 
\begin{itemize}
    \item $Q(0) = Q_0$,
    \item $Q(1) \in \underset{Q \in O_N(\R)}{\arg\min} \,\, \phi(Q)$,
    \item $\frac{d}{dt} \phi(Q(t)) \leq 0$.
\end{itemize}
\end{lemma}
\begin{proof}
First, we observe that the structure of all stationary and in particular, of global minima and saddle points is studied in  \citep[Theorem 4]{BROCKETT1989761}. From this last result all the minima are global minima, and all saddle points are strict under the hypothesis that all the eigenvalues of both $D_1$ and $D_2$ are distinct.
Let us now consider the Riemannian gradient flow
\[
\dot Q = - \mathrm{grad}\, \phi(Q) = -QD_2^2Q^\top D_1^2 Q+ D_1^2 Q D_2^2,
\]
where $\mathrm{grad}$ is the gradient induced by the Frobenius inner product on $O_N(\R)$.
We now use \citep[Theorem 6.3]{bah2022learning}, which guarantees almost sure (with respect to the initial condition) convergence of the gradient flow to the local minima. Since we know that all local minima are global, we get the desired thesis.

Moreover, the solution curve $Q(t)$ is analytic because the vector field $\mathrm{grad}\, \phi(Q)$ is analytic.
\end{proof}

Using \Cref{lemma:ff_lemma_svals}, it is possible to prove the main result:
\subsection{Proof of \Cref{thm:main_DNC_optimal_ff}}
\begin{proof}
Without loss of generality, we assume $C=1$ in \Cref{eq:constraint}. Moreover, we can exclude the case $Z^L \ne Y$, since in that case $\LimitLoss = + \infty$; therefore we assume $Z^L = Y$.
To understand the existence and the  structure of the minimizers, we can proceed iteratively.
To this end, 
 denote by $\mathcal S_l = \mathcal S_l(Z^{l-1},Z^{l+1})=\{Z^l \text{ that satisfy (I), (II) and (III)}\}$   for fixed, $Z^{l-1}$ and $Z^{l+1}$. 
 Note in particular that 
 \begin{equation}\label{eq1:thm_optimal_DNC1}
 \mathcal{S}_l(Z^{l-1},Z^{l+1})\subseteq \mathcal{S}_{loc}(\sigma_{l+1}^{-1}(Z^{l+1})),   
 \end{equation}
 where we refer to \eqref{eq:local_constraint} for the definition of $\mathcal{S}_{loc}(\sigma_{l+1}^{-1}(Z^{l+1}))$.
As a consequence we study the infimum in $Z^1$ and observe the following set of inequalities
 \begin{equation}\label{align1:thm_optimal_DNC1}
 \begin{aligned}
 \inf_{Z^1 \in \mathcal S_1} \LimitLoss(Z^L,\dots,Z^{1})&\geq \min _{Z^1 \in \mathcal{S}_{loc}(\sigma_{2}^{-1}(Z^{2}))} \LimitLoss(Z^L,\dots,Z^{1}) \\
 &= \frac{1}{2}\sum_{l=2}^{L-1}\|\sigma_{l+1}^{-1}(Z^{l+1})Z^{l,+} \|^2 + \frac{\rank(\sigma_2^{-1}(Z^2))}{2}  \\
 & \geq \frac{1}{2}\sum_{l=2}^{L-1}\|\sigma_{l+1}^{-1}(Z^{l+1})Z^{l,+} \|^2 + \frac{\rank(Y)}{2},
 \end{aligned}
 \end{equation}
 where we have used first \eqref{eq1:thm_optimal_DNC1}, second, \Cref{lemma:ff_lemma_svals}, which proves the existence of the minimum in $\mathcal{S}_{loc}(\sigma_{2}^{-1}(Z^{2}))$ and characterizes its value, and third, condition (III) in \Cref{eq:constraint}, which guarantees the last inequality when applied to $\rank(\sigma_{2}^{-1}(Z^{2}))$. Next, we focus on the lower bound in \eqref{align1:thm_optimal_DNC1} and define the following function, depending only on $Z^2,\dots Z^L$:
 \begin{equation*}
 \LimitLoss_{1}(Z^L,\dots,Z^2):=\frac{1}{2}\sum_{l=2}^{L-1}\|\sigma_{l+1}^{-1}(Z^{l+1})Z^{l,+} \|^2 + \frac{\rank(Y)}{2}.
 \end{equation*}
Thus, we proceed optimizing $\LimitLoss_1$ with respect to $Z^2$ as we did for $\LimitLoss$ with respect to $Z^1$ and obtain
 \begin{equation}\label{align3:thm_optimal_DNC1}
 \begin{aligned}
 \inf_{\substack{Z^1\in \mathcal{S}_1,\\ Z^2 \in \mathcal S_2}} \LimitLoss(Z^L,\dots,Z^2,Z^{1}) &\geq \inf_{Z^2 \in \mathcal S_2} \LimitLoss_1(Z^L,\dots,Z^2) \\
 &\geq \min_{Z^2\in \mathcal{S}_{loc}(\sigma_3^{-1}(Z^3))}\LimitLoss_1(Z^L,\dots,Z^2)\\
 &= \frac12 \sum_{l=3}^{L-1}\|\sigma_{l+1}^{-1}(Z^{l+1})Z^{l,+} \|^2 + \frac{\rank(\sigma_3^{-1}(Z^3))}{2} + \frac{\rank( Y)}{2} \\
 &\geq \frac12 \sum_{l=3}^{L-1}\|\sigma_{l+1}^{-1}(Z^{l+1})Z^{l,+} \|^2 + \frac{2\rank( Y)}{2}.
 \end{aligned}
 \end{equation}
 We continue iteratively by defining every time 
 \begin{equation*}
\LimitLoss_{i}(Z^L,\dots,Z^i):=     \frac12 \sum_{l=i+1}^{L-1}\|\sigma_{l+1}^{-1}(Z^{l+1})Z^{l,+} \|^2 + \frac{i\rank( Y)}{2}.
 \end{equation*}
 Finally, we obtain 
 \begin{equation*}
     \inf_{Z^i\in \mathcal{S}_i\; \forall i=1,\dots,L}\LimitLoss(Z^L,\dots,Z^2,Z^{1})\geq \frac{L}{2}\rank(Y).
 \end{equation*}

 Now we \textbf{claim} that there exists some $\bar Z\in \mathcal{S}$ such that $\LimitLoss(\bar Z)=(L/2)\rank(Y)$. Note that, as a consequence of the claim, we have that the minimum in $\mathcal{S}$ exists and that all the inequalities in \eqref{align1:thm_optimal_DNC1} and \eqref{align3:thm_optimal_DNC1} are equalities. In particular, because of \Cref{lemma:ff_lemma_svals}
 any minimum $\bar Z$ has to satisfy 
 \begin{equation}\label{eq2:thm_optimal_DNC1}
     \bar Z^l = O_l\sigma_{l+1}^{-1}(\bar Z^{l+1}) \qquad \forall l=1,\dots, L-1,
 \end{equation}
 for some sequence of matrices $O_l \in \R^{n_l \times n_{l+1}}$ such that $O_l^\top O_l = I$. We conclude by observing that $Y=Z^{L}$ satisfies DNC1 and so, applying \eqref{eq2:thm_optimal_DNC1} backwards from $L$ to $1$, also all the other layers $\bar{Z}^l$, at the minima, satisfy DNC1.

 We are left to prove the \textbf{claim}. To this end, recall that $y_j=e_{i_j}$ for any $j$ and $\sigma_l(0)=0$ for any $l$. So, $\sigma_{L}^{-1}(y_j)=\alpha_L e_{i_j}$ for any $j$ where $\alpha_L$ is a constant different from zero and independent of $j$. In particular we can consider $O_{L-1}^T=[\mathrm{Id},0]$ and $\bar{Z}^{L-1}=O_{L-1} \sigma_{L}^{-1}(Y)=\alpha[Y,0]^T$ and iteratively do the same for any $l$. The obtained $\bar{Z}$ is then included in $\mathcal{S}$ because each layer $Z^l$, as its counterimage $\sigma_l^{-1}(Z^l)$, are a rescaled version of $Y$. In particular, conditions \ref{cond1},\ref{cond2},\ref{cond3} are satisfied. Moreover, by its definition, $\|\sigma_{l+1}^{-1}(\bar{Z}^{l+1})\bar{Z}^{l,+}\|^2=\rank (\sigma_{l+1}^{-1}(\bar{Z}^{l+1}))=\rank(Y)$, yielding $\LimitLoss(\bar{Z})=(L/2)\rank(Y)$ and concluding the proof of the claim.
\end{proof}

\subsection{Reachability of DNC configuration for almost every initial condition}\label{subsec:ff_reachability}
A natural question arising from the last section is: from which initial configurations can the global minima be reached? In this section, we will prove not only that the $\DNC$ minima are global, but also that they are reachable from any interpolating initial condition in a related setting thanks to the result in \Cref{thm:main_reachability_nc}.

We are now ready for the main result about the reachability of rank-collapsed configurations through energy-decreasing paths.

\subsection{Proof of \Cref{thm:main_reachability_nc}}
\begin{proof}
    The idea is to construct a concatenation of curves layer by layer, where the existence of a curve for the local problem is given in \Cref{lemma:ff_lemma_svals}. Moreover, we notice immediately that according to \Cref{lemma:ff_lemma_svals}, this can be done for Lebesgue almost every $Z^l(0) \in \mathcal S$.
    
    We recall the notation $\LimitLoss(Z) = \sum_{l=1}^{L-1} \LimitLoss_l$, where $\LimitLoss_l := \frac12\|\sigma_{l+1}^{-1}(Z^{l+1})Z^{l,+} \|^2$. First observe that, from \Cref{lemma:ff_lemma_svals}, there exists a path $\gamma_1(t)$ moving $Z^1(0)$ to an optimal $Z^{1,*}(Z^2(0))=OZ^2(0)$, for some $O^\top O= I$.
    After having optimized the first layer, we can optimize the second one. To this end, 
    we consider a path $\gamma_2^2(t)$ moving $Z^2(0)$ to an optimum $Z^{2,*}(Z^3(0))$, which in turn induces a path $\gamma_2^1(t)\in Z^{1,*}(Z^2(t))$ on the first layer, which keeps it optimal.  
    For this second path $(\gamma_2^1(t), \gamma_2^2(t))$, thanks to \Cref{lemma:det_change_finitepoints} and the hypothesis that $\rank(\sigma_{2}^{-1}(Z^2(0))) = N$, we have that for all $t \in [0,1]$ but a finite set, $\rank( \sigma_2^{-1}(\gamma_2^2(t))) = N$, and therefore the function $\LimitLoss_2(\gamma_2^2(t)) + \rank(\sigma_2^{-1}(\gamma_2^2(t)))^2$ is continuous up to a finite number of removable discontinuities. In the discontinuities $t_i$, as shown in \Cref{lemma:det_change_finitepoints}, $\rank( \sigma_2^{-1}(\gamma_2^2(t_i))) < N$, the jump is down. In the other points, we have
    \[
    \frac{d}{dt}\LimitLoss(Z^L(0),\dots,\gamma_2^2(t),\gamma^1_2(t)) = \frac{d}{dt} \LimitLoss_2(\gamma_2^2(t)) + \frac{d}{dt}\rank(\sigma_2^{-1}(\gamma_2^2(t)))^2 = \frac{d}{dt} \LimitLoss_2(\gamma_2^2(t)) \leq 0,
    \]
    where the last is because the rank of $\sigma_2^{-1}(\gamma_2^2(t))$ is constant along the path $\gamma_2^2(t)$.
    
    Inductively, after having optimized the first $l-1$ layers, we can define paths 
$\gamma_l^{l'}(t)$ for all $l'\le l$ that move $Z^l(0)$ to an optimum while keeping all the previous layers optimal and satisfying the required decrease condition 
\[
\frac{d}{dt}\LimitLoss(Z^L(0),\dots,\gamma_l^l(t),\dots, \gamma_l^2(t),\gamma_l^1(t)) = \frac{d}{dt} \LimitLoss_l(\gamma_l^l(t)) \leq 0.
\]
By concatenating all of the paths that we have considered we obtain the required path $\gamma$.
Note that for the last path, the rank of all layers stays constant for almost all times, while collapsing to $K$ at $t = 1$, therefore again decreasing the objective. 
\end{proof}
\begin{lemma}[Finitely-many zeros of determinant on analytic matrix curves]\label{lemma:det_change_finitepoints}
Consider an entrywise analytic function $g$ and a piecewise analytic curve $A: [0,1] \to \mathbb R^{N \times N}$ with $\rank g(A(0)) = N$. Then there exists at most a finite number of times $\{t_1,\dots,t_k\}$ in which such $\rank g(A(t)) \leq N-1$. In particular, the function $t \mapsto \rank(g(A(t)))$ has just removable discontinuities, i.e., left and right limits exist and $\underset{t \to t_i^+}{\lim} \rank(g(A(t))) = \underset{t \to t_i^-}{\lim} \rank(g(A(t))) = N> \rank(g(A(t_i)))$.
\end{lemma}
\begin{proof}
We can prove it without loss of generality for a fully analytic curve $A(t)$, and then the same result holds on all subintervals in which $A$ is analytic (note that there are finitely many of them).

The function $t \mapsto \det(g(A(t)))$ is analytic and not constantly zero since $\det g(A(0))\ne 0$, and therefore, thanks to \Cref{lemma:zero_analytic_functions}, it has a finite number of zeros in $[0,1]$. Therefore, there can be just a finite number of points in which $\rank g(A(t)) \leq N-1$.
Now, we will prove that the discontinuities are removable. In particular, let $\varepsilon>0$. Since $\rank(g(A(t))) = N$ for all $t \in [0,1]\setminus \{t_1,\dots,t_k \}$, there exists $\delta>0$ such that no $t_j \in (t_i-\delta,t_i)$. In $(t_i-\delta,t_i)$, $\rank(g(A(t)) \equiv N$ and therefore we proved that the left limit exists and it is equal to $N$. The same can be done for the right limit.
Therefore the concatenation of the paths $\gamma$ is continuous and $\LimitLoss \circ \gamma$ has a finite number of discontinuities $E :=\{t_i\}_i$.
\end{proof}

As a corollary of \Cref{thm:main_reachability_nc} we have the following:
\subsection{Proof of \Cref{cor:finite_step_reachability}}\label{proof:finite_step_reachability}
\begin{proof}
    Consider the curve $\gamma$ from \Cref{thm:main_reachability_nc} together with its finite number of discontinuities $E :=\{t_i\}_i$ in which
 \[
 \underset{t \to t_i^-}{\lim} \LimitLoss(\gamma(t)) = \underset{t \to t_i^+}{\lim} \LimitLoss(\gamma(t)) >\LimitLoss(\gamma(t_i)).
 \]
 Consider now a sequence of sets $(\{T_{j,k} \}_{j=0,\dots,k-1})_k$, where $0<T_{j+1,k}-T_{j,k} \leq \frac{1}{k}$ for all $j$, $T_{0,k} = 0, T_{k-1,k} = 1$ for all $k$, and such that $T_{j,k} \notin E$ for all $j,k$.
 Define now the sequence of paths
 \[
 \gamma_k(s) = \gamma\Bigl( \sum_{j=0}^{k-1} T_{j,k} 1_{[T_{j,k},T_{j+1,k})}(s) \Bigr).
 \]
 Then, by construction, we have $\gamma_k : [0,1] \mathcal \to S$, $\gamma_k(0) = Z(0), \gamma_k(1)$ global minimizer. Moreover, since $\gamma_k(T_{j,k}) = \gamma(T_{j,k})$ and $T_{j,k} \notin E$ we have that the sequence $(\LimitLoss(\gamma_k(T_{j,k})))_j = (\LimitLoss(\gamma(T_{j,k})))_j$ is decreasing.
 Finally, we show uniform convergence to the $\gamma$ constructed above, by using the fact that $\gamma$ is piecewise analytic and therefore Lipschitz:
 \begin{align*}
 \sup_{s \in [0,1]} \| \gamma_k(s)-\gamma(s) \| & = \max_{j=0,\dots,k-1} \sup_{s \in [T_{j,k},T_{j+1,k}]} \|\gamma_k(s)-\gamma(s) \| \\ 
 &= \max_{j=0,\dots,k-1} \sup_{s \in [T_{j,k},T_{j+1,k}]} \|\gamma(T_{j,k})-\gamma(s) \|  \\
 & \leq L \max_{j=0,\dots,k-1}|T_{j+1,k}-T_{j,k}| \leq \frac{L}{k} \to 0. & (\gamma \text{ Lipschitz})
 \end{align*}
 Therefore $\gamma_k$ converges to $\gamma$ uniformly.
\end{proof}

\begin{remark}[Intuition on \Cref{cor:finite_step_reachability}]\label{remark:reachability_finite_step}
    We highlight that an interpretation of \Cref{cor:finite_step_reachability}, is in terms of numerical schemes. In particular, since $\gamma_k:[0,1] \to \mathcal S$ is piecewise constant, we can think of it as an ordered sequence of $k$ points $\Gamma_k:= (\gamma_k(0),\gamma_k(T_{1,k}),\dots,\gamma_k(T_{k-1,k}),\gamma_k(1))$, where $\gamma_k(0) = Z(0)$ and $\gamma_k(1) \in \underset{Z \in \mathcal S}{\arg\min}\,\LimitLoss$. We know from \Cref{cor:finite_step_reachability} that $0<T_{j+1,k}-T_{j,k}:= \Delta_{j,k}<\frac{1}{k}$, and we can think of $\Delta_{j,k}$ as the stepsize of an optimization scheme. With this heuristic intuition in mind, the result proven in the corollary essentially states that for Lebesgue-almost every initialization $Z(0) \in \mathcal S$, there exists a numerical scheme with \textbf{arbitrarily small} maximal stepsize such that $\LimitLoss$ decreases along it.
    Notice that, if there was any loss barrier between $Z(0)$ the the global minimizers of $\LimitLoss$ on $\mathcal S$, then this would not be possible: in particular, if there was a barrier between initialization and minimizers, then there would exist a small enough stepsize such that $\LimitLoss$ would need to increase along the sequence.
\end{remark}

\subsection{Stability of the $\DNC$ Configurations and Proof of \Cref{prop:main_gammaconv_firstlayer}}\label{subsec:stability_ff}
\begin{proof}
We can without loss of generality assume $Z^L = Y$, otherwise the condition is trivially satisfied. Therefore, we can assume $\mathcal S$ contains the interpolating condition $Z^L = Y$ as a constraint.
    Define and consider a sequence $Z_\lambda \to Z^*  \in \mathcal S$ for $\lambda \to \infty$, $Z_\lambda \in \mathcal S$. Since $\mathcal S$ contains the interpolating condition, we have $\LimitLoss_\lambda(Z^*) < +\infty$.  Define $M:= \lim \inf_{\lambda \to +\infty} \LimitLoss_\lambda(Z_\lambda) = \lim_{m \to +\infty} \LimitLoss_{\lambda_m}(Z_{\lambda_m})$ for a converging subsequence $\lambda_m$ (we can exclude the case $M = + \infty$ cause the required inequality is trivially satisfied). Then, thanks to the lower semicontinuity of $\LimitLoss_\lambda$ we get
    \[
   \LimitLoss(Z^*) \leq \mathcal \liminf_{m} \LimitLoss(Z_{\lambda_m}) = \liminf_{m \to \infty} \frac12\sum_{l=1}^{L-1} \|\sigma_{l+1}^{-1}(Z^{l+1}_{\lambda_m})Z^{l,+}_{\lambda_m} \|^2 \leq \liminf_m \LimitLoss_{\lambda_m}(Z_{\lambda_m}) < + \infty.
    \]
    As a recovery sequence for the $\limsup$ condition we take $Z_\lambda \equiv Z^*$, in fact we have
    \[
    \limsup_{\lambda \to \infty} \LimitLoss_{\lambda}(Z^*) = \lim_{\lambda \to \infty} \LimitLoss_{\lambda}(Z^*) = \LimitLoss(Z^*).
    \]
    Moreover, let $\lambda_n \to +\infty$ with $Z_n \in \arg\min_{Z\in \mathcal{S}} \LimitLoss_{\lambda_n}(Z)$ and that it admits an accumulation point $Z^* \in \mathcal S$. 
    Then there exists a convergent subsequence $Z_{n_j} \to Z^*$.
    For every $(\mathscr Z_j)$ recovery sequence for a generic $\mathscr Z$ we have, for the definition of $\Gamma$-convergence and the fact that $Z_{n_j}$ is a minimizer of $\LimitLoss_{\lambda_{n_j}}$,
    \[
    \LimitLoss(Z^*) \leq \lim\inf_{j} \LimitLoss_{\lambda_{n_j}}(Z_{n_j}) \leq \lim\inf_{j} \LimitLoss_{\lambda_{n_j}}(\mathscr Z_{j}) \leq \lim\sup_{j} \LimitLoss_{\lambda_{n_j}}(\mathscr Z_{j}) \leq \LimitLoss(\mathscr Z),
    \]
    which implies $\LimitLoss(Z^*) \leq \LimitLoss(\mathscr Z)$ for every $\mathscr Z \in \mathcal S$, therefore $Z^*$ is a minimizer.
\end{proof}

\section{Global Minima of Representation Cost for Residual Networks} \label{sec:resnet_global}
For residual networks, the representation of the layer is given by
\[
Z^{l+1} = \sigma_{l+1}(W_{l+1}Z^l) + Z^l
\]
and therefore we have
\[
W_{l+1} = \sigma_{l+1}^{-1}(Z^{l+1}-Z^l)Z^{l,+},
\]
giving the representation cost
\[
\LimitLoss^{res}(Z) := \begin{cases}
    +\infty, \quad \text{if}\,\, f(\Theta;X) \ne Y \\
    \frac12\|\sigma_{L}^{-1}(Z^{L})Z^{L-1,+}\|^2+\frac12\sum_{l = 2}^{L-2} \|\sigma_{l+1}^{-1}(Z^{l+1}-Z^l)Z^{l,+}\|^2,\quad \text{otherwise}.
\end{cases}
\]
The proof of the main result relies on the following lemma:
\begin{lemma}[Optimal single layer for ResNet representation cost]\label{lemma:resnet_lemma}
    Let $\localF^{res}(B) = \frac{1}{2}\| g(A-B)B^+\|^2$, where $A \in \mathbb R^{n \times N}$ is a fixed matrix and $g:\mathbb R \to \mathbb R$ is entrywise nonlinear with $g(x) = 0 \iff x = 0$.
    Then, for $n \leq m$, we have that
    \begin{align*}
    &\min_{B \in \mathbb{R}^{m \times N}}\localF^{res}(B)
    \quad  \textnormal{s.t.}\quad g(A-B)B^+B = g(A-B) 
    \end{align*}
    is attained at the point $B^* = A$.
\end{lemma}
\begin{proof}
    Notice that $\localF^{res}(B) \geq 0$ and $\localF^{res}(B) = 0$ if and only if $g(A-B)B^+ = 0$. Then we have because of the constraint that 
    \[
    0 = g(A-B)B^+B = g(A-B).
    \]
    So the only point attaining the minimum satisfies $g(A-B) = 0$ which implies $B^* = A$.
\end{proof}

In this setting we have the following result about global minima:
\begin{theorem}[$\DNC$ is optimal for constrained representation cost]\label{thm:NC1_optimal_resnet}
Let $\sigma_l$ be a sequence of entrywise nonlinearities for which $\sigma_l(x) = 0 \iff x = 0$ and assume intermediate widths satisfy the condition $K \leq n_L = n_{L-1}  = \dots = n_1 = N$.
Then, for almost every $Z^1 = \sigma_1(W_1X)$, the global minima of the optimization problem
\begin{align*}
&\min_{Z^2,\dots,Z^L} \LimitLoss^{res}(Z) \\
&\quad \textnormal{s.t.} \quad \sigma_{l+1}^{-1}(Z^{l+1}-Z^l)Z^{l,+} Z^l = \sigma_{l+1}^{-1}(Z^{l+1}-Z^l), \quad \forall \, l = 1,\dots,L-2 \\
&\quad  \qquad\, \sigma_{L}^{-1}(Z^L)Z^{L-1,+}Z^{L-1} = \sigma_{L}^{-1}(Z^L) \\
&\quad \qquad\, s_i(Z^{L-1}) \leq C s_i(\sigma_L^{-1}(Y)),\quad \forall i = 1\,\dots,r_L
\end{align*}
satisfy $\DNC$ for all intermediate layers $l = 1,\dots,L$.
\end{theorem}
\begin{proof}
Let's define with $\mathcal S$ the full constraint set and with $\mathcal S_l = \mathcal S_l(Z^1,\dots,Z^{l-1},Z^{l+1},\dots,Z^L)$ the constraint set on $Z^l$ with the other variables fixed.
    Notice that we can first minimize in $Z^2$ the problem
    \begin{align*}
    &\arg\min_{Z^2 \in \mathcal S_2} \LimitLoss^{res}(Z^L,\dots,Z^1) = \arg\min_{Z^2 \in \mathcal S_2}\|\sigma_3^{-1}(Z^3-Z^2)Z^{2,+} \|^2 \\
    & \qquad\; \text{s.t.} \quad \sigma_{3}^{-1}(Z^{3}-Z^2)Z^{2,+} Z^2 = \sigma_{3}^{-1}(Z^{3}-Z^1)
    \end{align*}
    whose minima is given by \Cref{lemma:resnet_lemma} at $Z^{2,*}(Z^3) = Z^3$ and minimal value equal to zero.
    Therefore, we can now minimize in $Z^3$ the representation cost
    \begin{align*}
&\min_{Z^3,\dots,Z^L} \LimitLoss^{res}(Z^L,\dots,Z^2,Z^{2,*}(Z^3)) = \frac12\|\sigma_{L}^{-1}(Z^{L})Z^{L-1,+}\|^2+ \frac12\sum_{l=3}^{L-1}\| \sigma_{l+1}^{-1}(Z^{l+1}-Z^l)Z^{l,+} \|^2\\
&\text{s.t.} \quad \sigma_{l+1}^{-1}(Z^{l+1}-Z^l)Z^{l,+} Z^l = \sigma_{l+1}^{-1}(Z^{l+1}-Z^l) \quad \forall \, l = 3,\dots,L-1 \\
& \qquad\, \sigma_{L}^{-1}(Z^L)Z^{L-1,+}Z^{L-1} = \sigma_{L}^{-1}(Z^L) \\
& \qquad\, s_i(Z^{L-1}) \leq C s_i(\sigma_L^{-1}(Y)),\quad \forall i = 1\,\dots,K.
\end{align*}
Notice that the situation now is the same as the initial one, but with one layer less.
Therefore, we can do the same thing  for $Z^3$ using \Cref{lemma:resnet_lemma} (with $g = \sigma_l^{-1}$) to obtain $Z^{3,*}(Z^4) = Z^4$. By iterating this procedure for all layers, we arrive at
\begin{align*}
&\min_{Z^{L-1},Z^L} \frac12 \|\sigma_{L}^{-1}(Z^L)Z^{L-1,+} \|^2 \\
& \quad \text{s.t.} \quad Z^L = Y \\
& \quad \qquad\, \sigma_{L}^{-1}(Z^L)Z^{L-1,+}Z^{L-1} = \sigma_{L}^{-1}(Z^L) \\
& \quad \qquad\, s_i(Z^{L-1}) \leq C s_i(\sigma_L^{-1}(Y)),\quad \forall i = 1\,\dots,K.
\end{align*}
By using \Cref{lemma:ff_lemma_svals} we get that 
\[
Z^{L-1,*} = O_l \sigma_L^{-1}(Y), \quad O_l^\top O_l = I.
\]
Therefore, $Z^{L-1,*}$ satisfies $\DNC$ cause $\tilde Y = \sigma_{L}(Y)$ satisfies $\DNC$. Recursively backward, we have by \Cref{lemma:resnet_lemma} that
\[
Z^{L-2,*} = Z^{L-1,*},
\]
which is therefore also collapsed. In the same way, by backsubstituting the optimal point, we get
\begin{align*}
    &Z^{L-1,*} \quad \text{collapsed}\\
    &Z^{L-2,*} = Z^{L-1,*}\\
    & \vdots \\
    & Z^{1,*} = Z^{2,*}
\end{align*}
and therefore
\[
Z^{L-1,*} = Z^{L-2,*} = \dots = Z^{1,*}
\]
and so they all satisfy $\DNC$.
\end{proof}

\begin{remark}(Stability of minima for ResNets)\label{remark:resnet_stability}
We highlight that the same $\Gamma$-convergence result as the one presented in \Cref{prop:main_gammaconv_firstlayer} can be stated also for ResNets, and the proof is exactly the same as the one presented in \Cref{subsec:stability_ff}.
More precisely, given $\phi:\mathcal S \to \R_+$ lower semicontinuous, we define
\begin{equation*}
\LimitLoss_{\lambda}^{res}(Z) = \begin{cases}
\frac12\|\sigma_{L}^{-1}(Z^{L})Z^{L-1,+}\|^2+\frac12\sum_{l = 1}^{L-2} \|\sigma_{l+1}^{-1}(Z^{l+1}-Z^l)Z^{l,+}\|^2 + \frac{\phi(Z)}{\lambda} , \,\,\text{if}\,\, Z^L = Y, \\
    +\infty, \quad \text{otherwise}. \\
\end{cases}
\end{equation*}
Then for $\lambda \to +\infty$, $\LimitLoss_\lambda^{res} \to \LimitLoss^{res}$ in the $\Gamma$-convergence sense, analogously to \Cref{prop:main_gammaconv_firstlayer}.
This will also be showcased numerically in \Cref{sec:additional_numerics}.
\end{remark}

\section{Additional Numerical Experiments}\label{sec:additional_numerics}

In this section, we include additional numerical results. All experiments were performed on a Single NVIDIA A100 (80GB) GPU.


\subsection{Residual Networks}

In this section we will show numerical results concerning \Cref{thm:NC1_optimal_resnet}. The setting is the same of \Cref{sec:num-exp}, with $L = 6$ layers (first and last feedforward, and the $4$ in the middle are residual). From \Cref{fig:resnet_weights}, we can observe the evolution of the norm of the weights during optimization, also for varying $\delta = \frac{1}{\lambda}$ in \Cref{eq:loss_regularized_on_1st_layer}. As we can see, for $\delta$ going to zero, all the weight matrices converge to zero as expected from \Cref{thm:NC1_optimal_resnet}, but the first and the last ones, which are feedforward layers.

\begin{figure}[t]
\centering
\begin{tabular}{cc}
 $\delta = 8$ &
 $\delta = 5$ \\
  \includegraphics[width=0.4\textwidth]{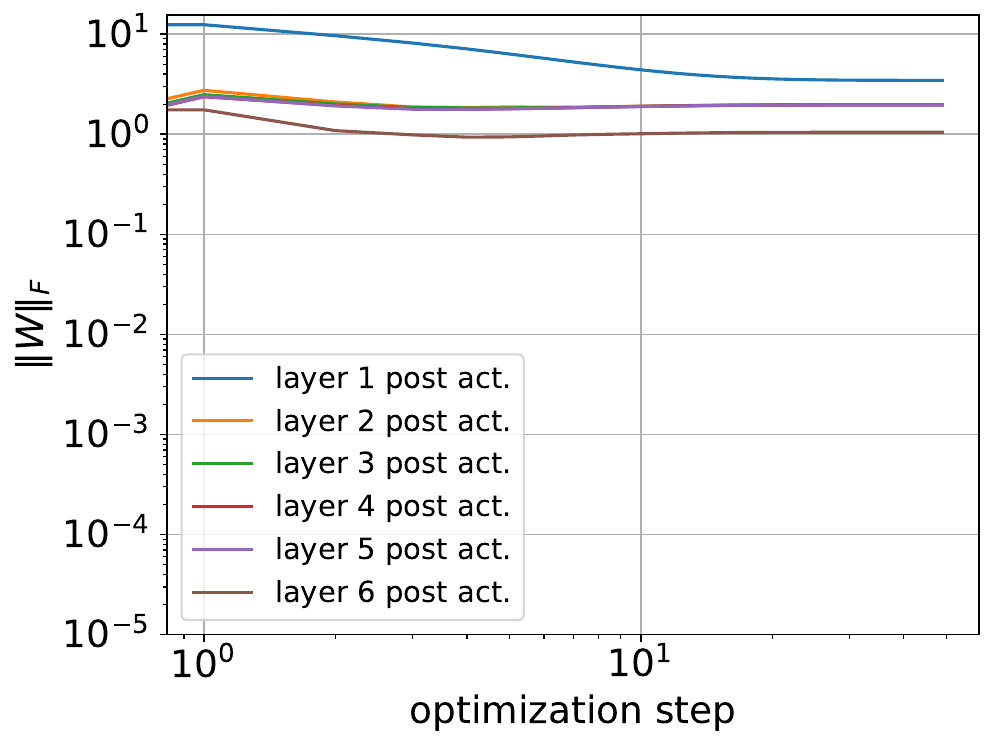} &   \includegraphics[width=0.4\textwidth]{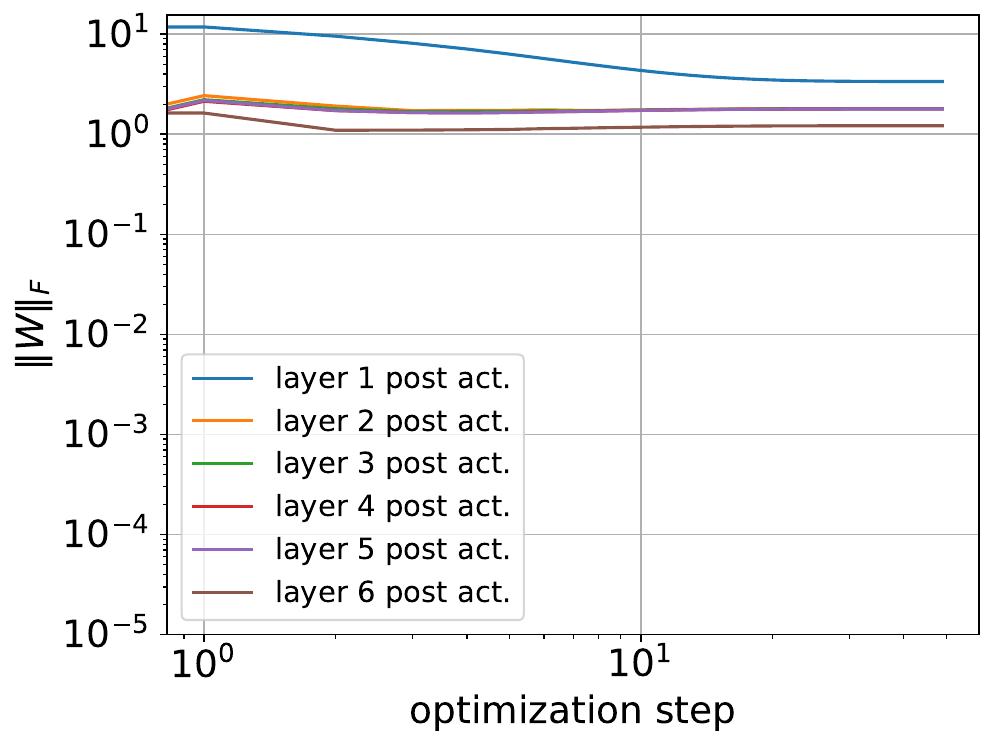} \\
  $\delta = 10^{-1}$ &
 $\delta = 10^{-2}$ \\
  \includegraphics[width=0.4\textwidth]{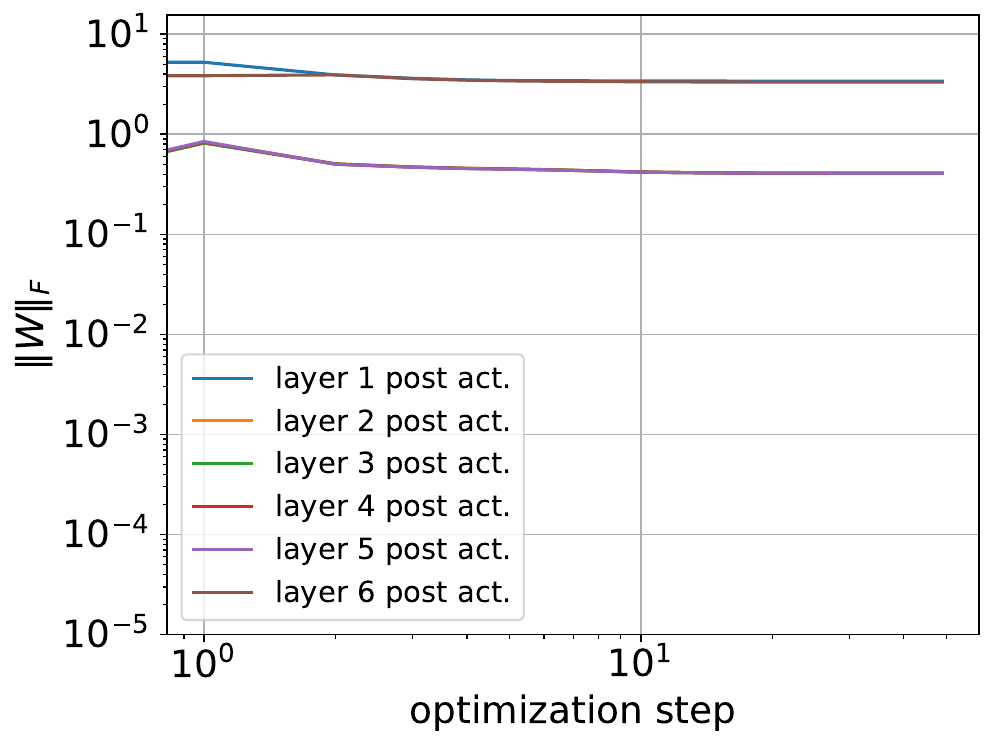} &   \includegraphics[width=0.4\textwidth]{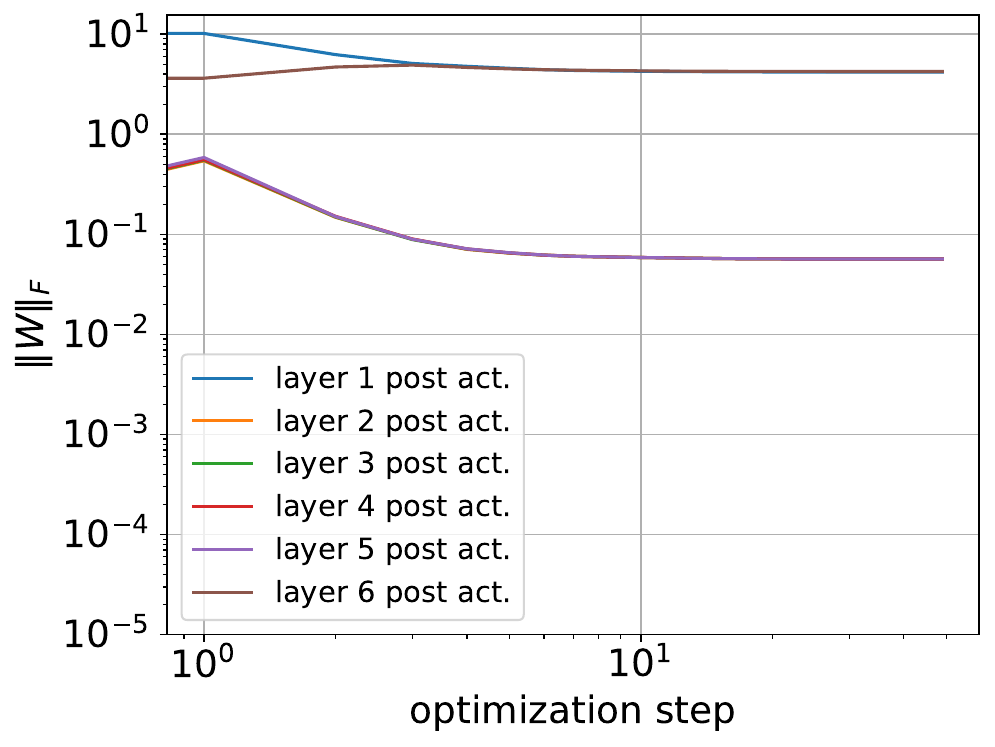} \\
  $\delta = 10^{-3}$ &
 $\delta = 0$ \\
  \includegraphics[width=0.4\textwidth]{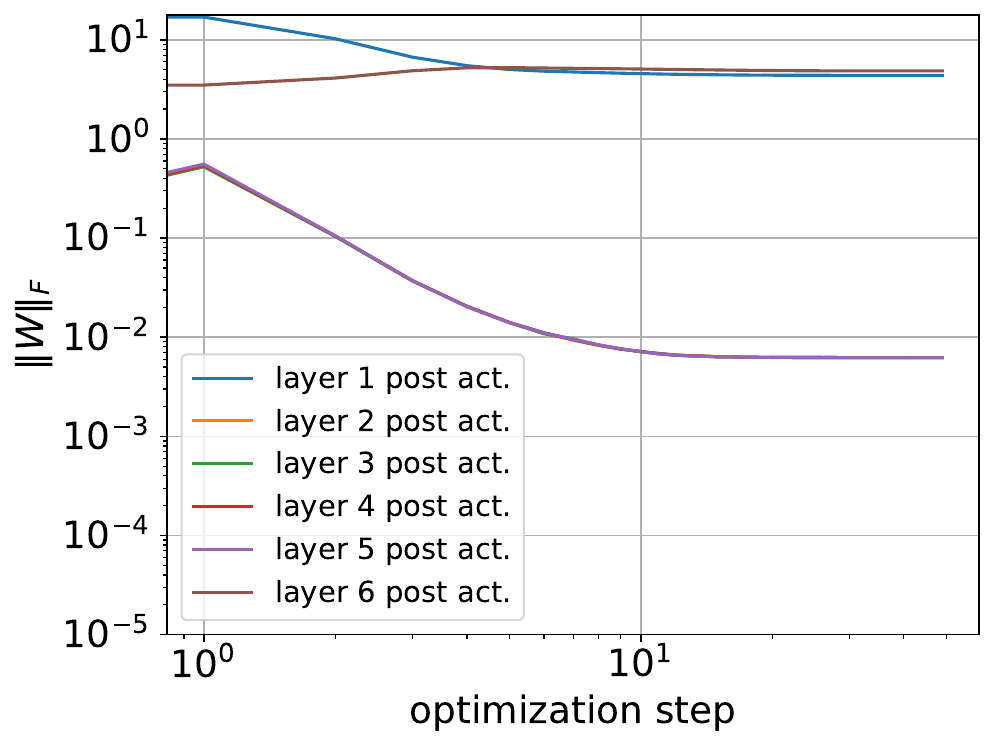} &   \includegraphics[width=0.4\textwidth]{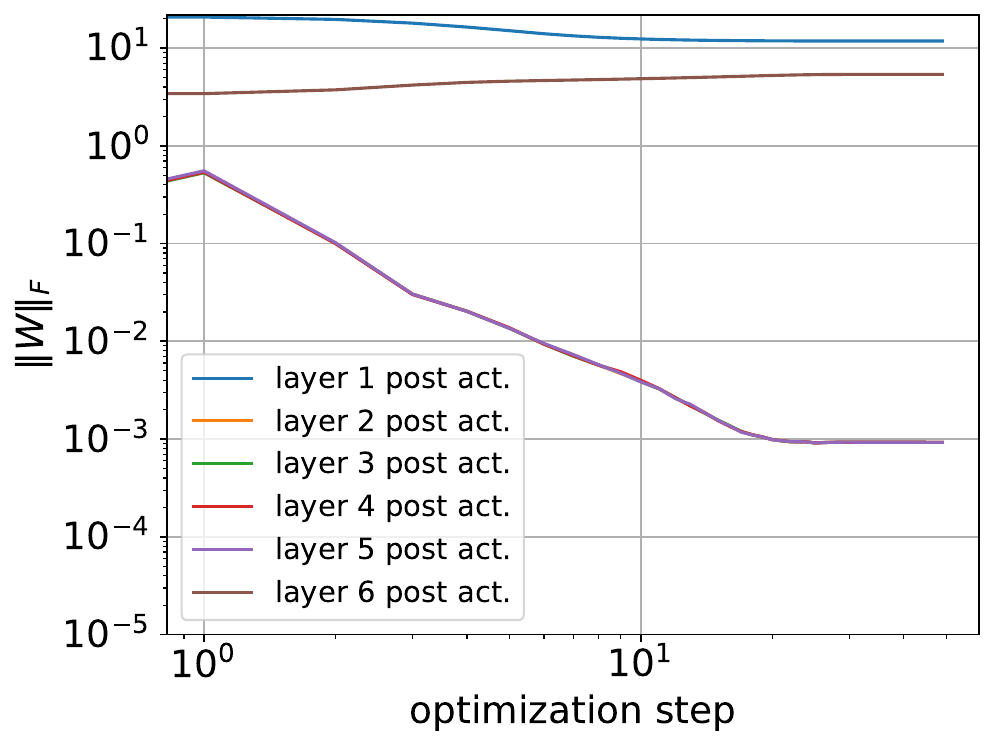} \\
\end{tabular}
\caption{Norm of the intermediate weights matrices for a for $L = 6$ layer ResNet during training. As discussed in \Cref{remark:resnet_stability}, we observe convergence to the structured minima discussed in \Cref{thm:NC1_optimal_resnet} for $\delta = \frac1\lambda$ (as in \Cref{eq:loss_regularized_on_1st_layer}).}
\label{fig:resnet_weights}
\end{figure}

\subsection{Visual Representation of Neural Collapse}

In this section we will include the evolution of the gram matrices $G^l = Z^{l,\top}Z^l$ for all layers at different phases of training. Notice that $\DNC$ happens if and only if $G^l$ has block structure. In fact, we have the following characterization:

\begin{proposition}[Characterization of $\DNC$ features through covariance]
    The following are equivalent:
    \begin{enumerate}
        \item $Z \in \mathbb{R}^{n \times N}$ satisfied NC1;
        \item $G = Z^\top Z = \left[ 
\begin{array}{c|c|c|c|c} 
  \alpha_{11}  1_{N_1}  1^\top_{N_1} & \alpha_{12}  1_{N_1}  1^\top_{N_2} & \alpha_{13}  1_{N_1}  1^\top_{N_3} & \dots &\alpha_{1K}  1_{N_1}  1^\top_{N_K} \\
  \hline
  \alpha_{21}  1_{N_2}  1^\top_{N_1} &  \alpha_{22}  1_{N_2}  1^\top_{N_2} & \alpha_{23}  1  1^\top & \dots &\alpha_{2K}  1  1^\top\\
  \hline
  \vdots &\dots &\ddots & \dots &\alpha_{K-1,K}  1  1^\top \\
  \hline
  \alpha_{K1}  1_{N_K}  1^\top_{N_1} & \alpha_{K,2}  1  1^\top & \alpha_{K,3}  1  1^\top & \dots &  \alpha_{K,K}  1  1^\top 
\end{array} 
\right]$ has block structure.
    \end{enumerate}
\end{proposition}
\begin{proof}
For simplicity, throughout the proof, we will denote with $z_i^k$ the $i$-th vector of the $k$-th class.
    Assume first that $Z$ satisfies NC1.  Then we have that $z_i^k=z^k$ for all $i=1,\dots, N_k$ and
    \[
    Z^\top Z = \left[ 
\begin{array}{c}
1_{N_1} z^{1,\top} \\
1_{N_2} z^{2,\top} \\
\vdots \\
1_{N_K} z^{K,\top} \\
\end{array}
\right]
\left[ 
\begin{array}{c}
z^1 1_{N_1}^\top, 
z^2 1_{N_2}^\top, 
\dots 
z^K 1_{N_K}^\top
\end{array}
\right] = \Bigl(\langle z^i, z^j \rangle  1_{N_i} {1}_{N_j}^\top\Bigr)_{i,j = 1,\dots,K},
    \] 
    which has the required structure of $Z^\top Z$ with $\alpha_{ij} = \langle z^i, z^j \rangle$.
    Vice versa, assuming that $Z^\top Z$ has the required structure, we have
    \[
    \alpha_{kk} = \langle z_i^k, z_j^k \rangle,\quad \forall i,j = 1,\dots,N_k \;\text{and}\quad \forall k = 1,\dots,K.
    \]
    This implies that
    \[
    0 = \langle z_i^k, z_j^k - z_{m}^k \rangle, \quad \forall i,j,m = 1,\dots,N_k.
    \]
    This, in turn, implies by definition of orthogonal projection that
    \[
    z_i^k = P^\perp_{\mathbb A^k}(0),\quad \forall i = 1,\dots,N_k\; \text{and}\quad \forall k  =1,\dots,K,
    \]
    where $\mathbb A_k = \text{Aff}(\{z_j^k\}_j)$ is the affine subspace passing through the points of class $k$.
    By uniqueness of the orthogonal projection we get
    \[
    z_i^k = z_j^k = P_{\mathbb A^k}^\perp (0) \quad \forall i,j = 1,\dots,N_k\; \text{and}\quad \forall k=1,\dots,K,
    \]
    which means $Z$ satisfies NC1.
\end{proof}
Given this premise, in the following experiments we will plot the Gram matrices $G_l=Z^{l,\top}Z^l$ for any $l=0,\dots,L-1$ in the case of a $L = 15$ layer neural network trained in the setting of \Cref{sec:num-exp}. The emergence of the block structure is depicted in \Cref{fig:block_structure}.

Moreover, we sampled a random plane in each feature space, i.e., for any $l$ we sample a random orthogonal matrix $B^l \in \mathbb R^{2 \times n_l}, B^lB^{l,\top} = I_2$. \Cref{fig:random_proj} shows the projection of the feature matrices $Z^l$ on the plane spanned by the rows of $B^l$, i.e. $B^l Z^l$, for which we plot all the columns as points in $\mathbb R^2$, allowing us to pictorially see the deep-neural (rank) collapse in all intermediate layers in this random subspace.

\begin{figure}[t]
\centering
\begin{tabular}{cc}
Epoch 0 & 
Epoch 100 \\
  \includegraphics[width=0.4\textwidth]{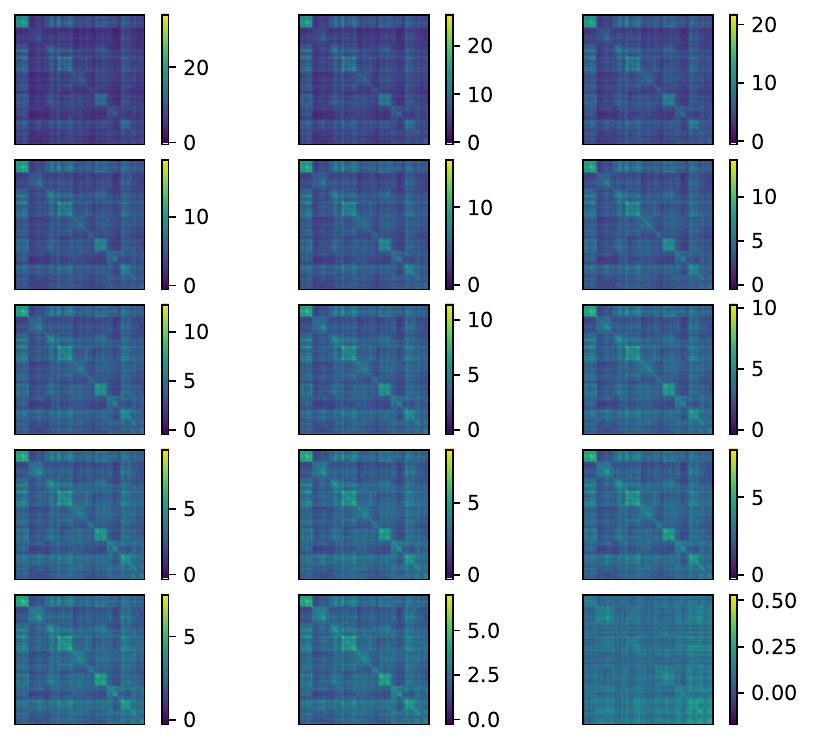} &   \includegraphics[width=0.4\textwidth]{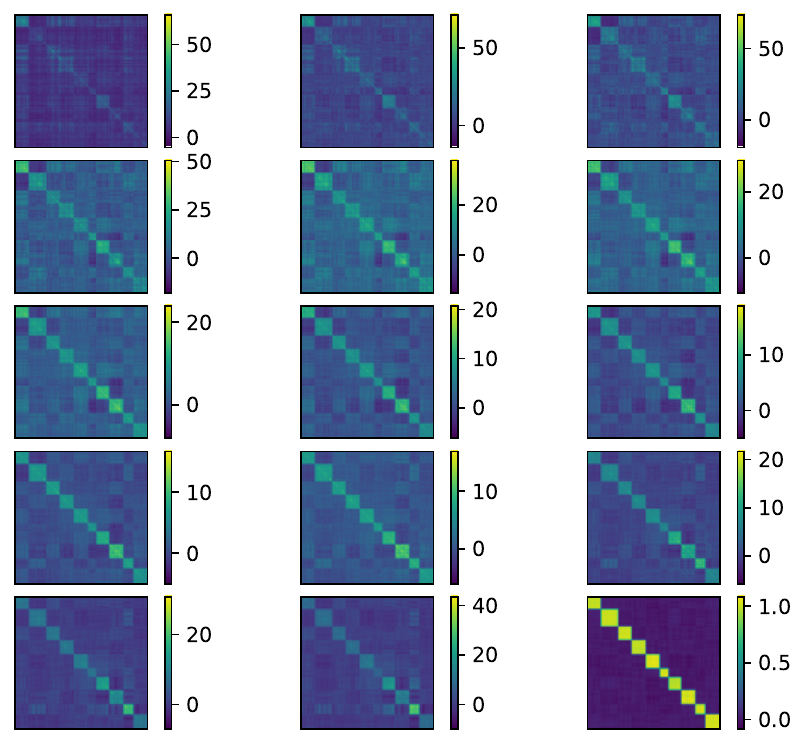} \\
  Epoch 200 & 
Epoch 400 \\
  \includegraphics[width=0.4\textwidth]{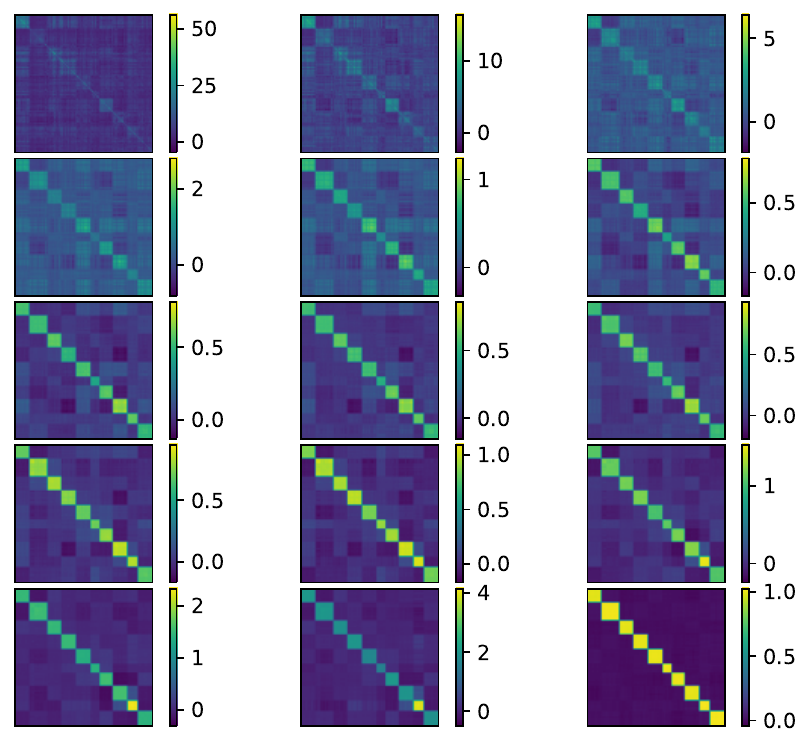} &   \includegraphics[width=0.4\textwidth]{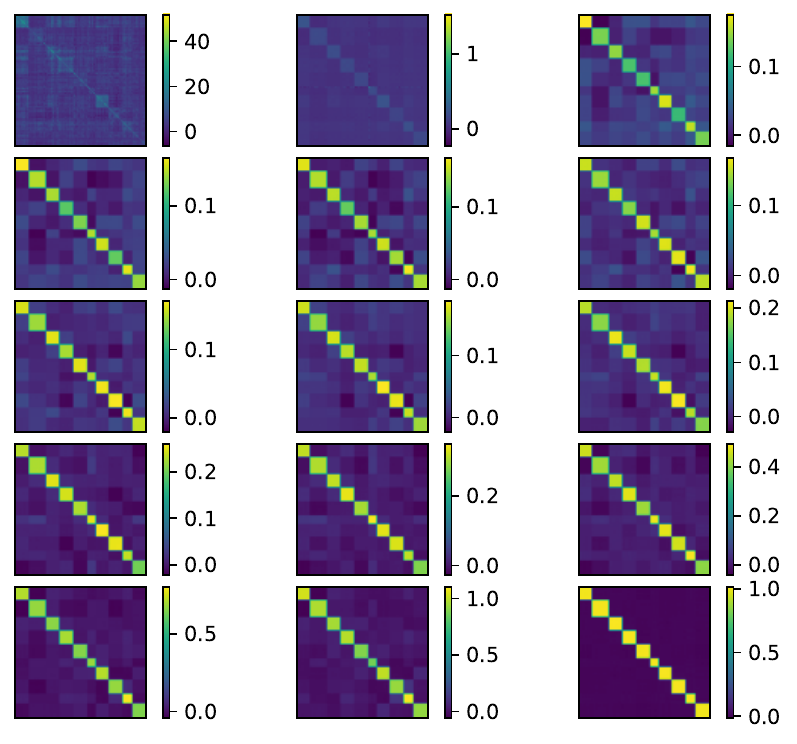} \\
  Epoch 700 & 
Epoch 1200 \\
  \includegraphics[width=0.4\textwidth]
  {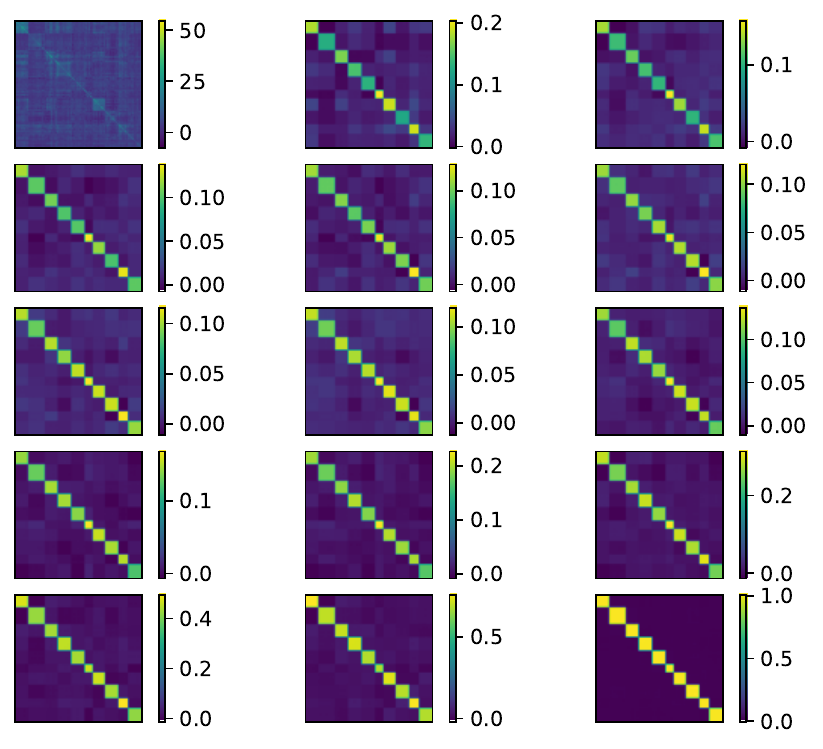} &   \includegraphics[width=0.4\textwidth]{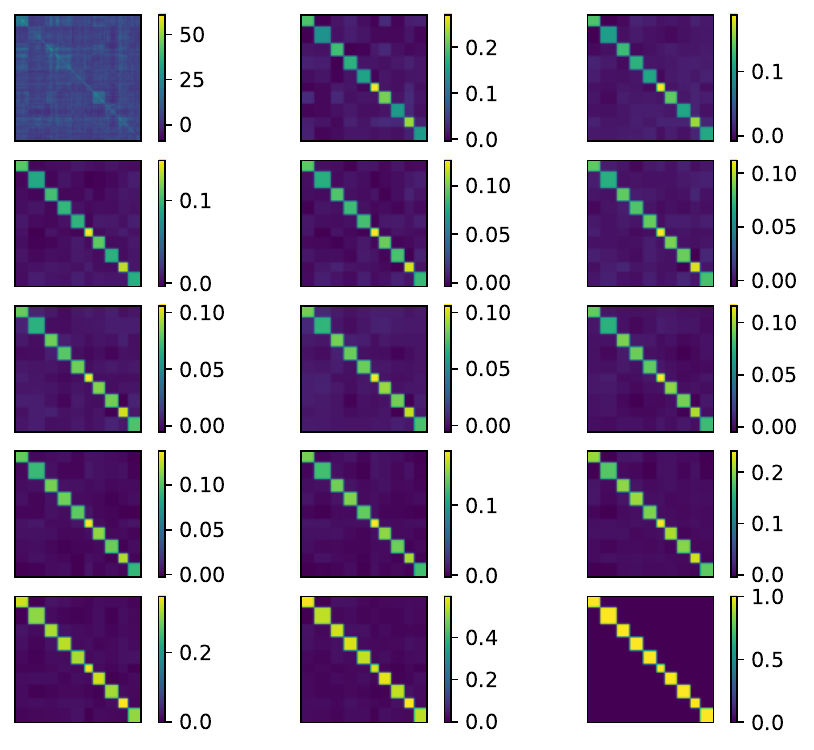} \\
\end{tabular}
\caption{Emergence of the block structure in $Z^{l,\top}Z^l$ characterizing $\DNC$ during training. Layers for each epoch are ordered from left to right and from top to down.}
\label{fig:block_structure}
\end{figure}

\begin{figure}[t]
\centering
\begin{tabular}{cc}
Epoch 0 & 
Epoch 100 \\
  \includegraphics[width=0.4\textwidth]{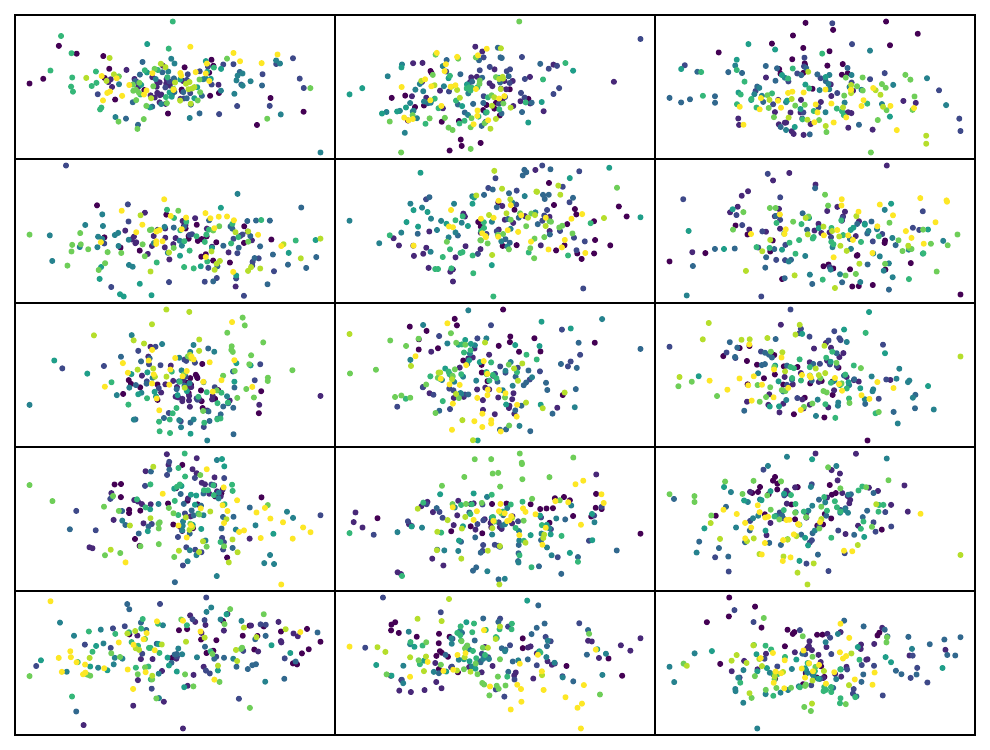} &   \includegraphics[width=0.4\textwidth]{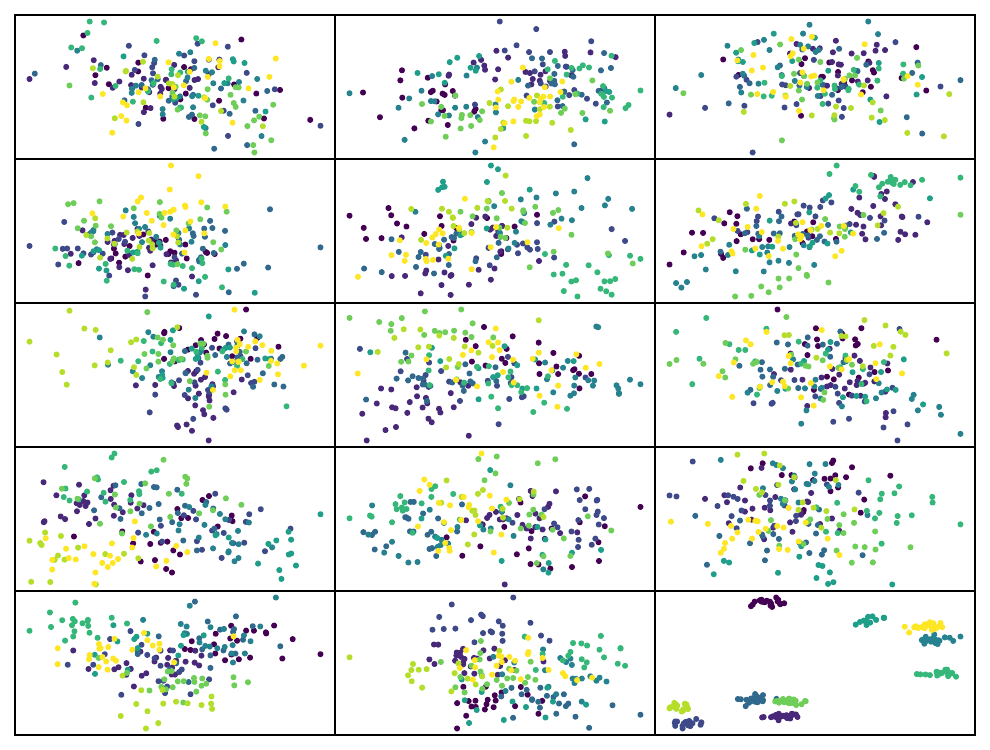} \\
  Epoch 200 & 
Epoch 300 \\
  \includegraphics[width=0.4\textwidth]{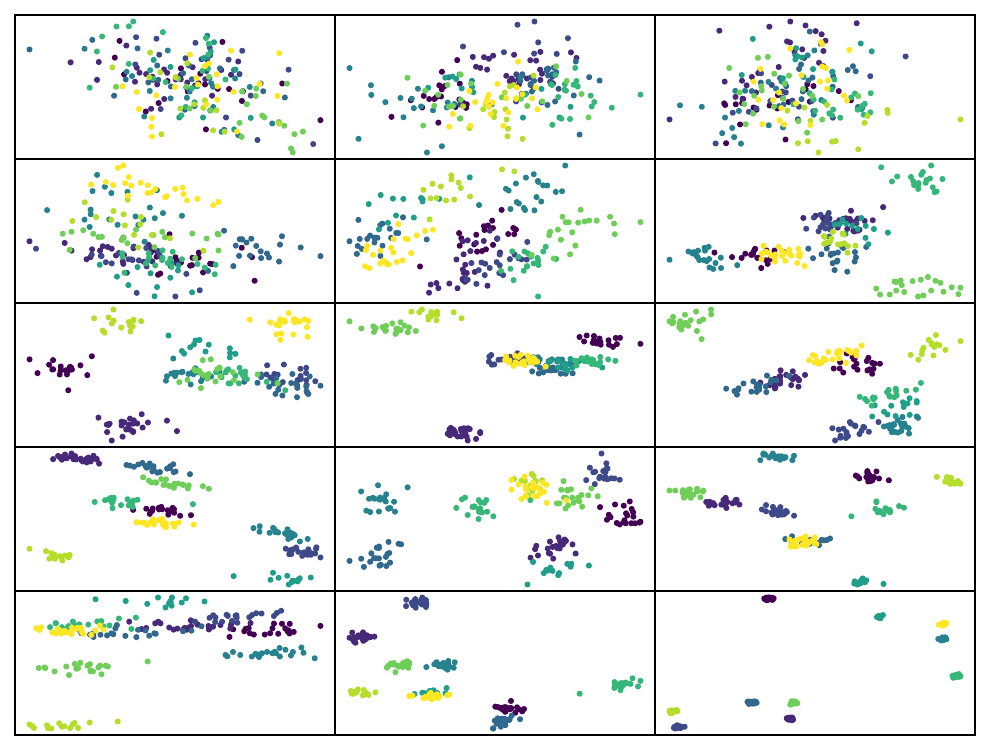} &   \includegraphics[width=0.4\textwidth]{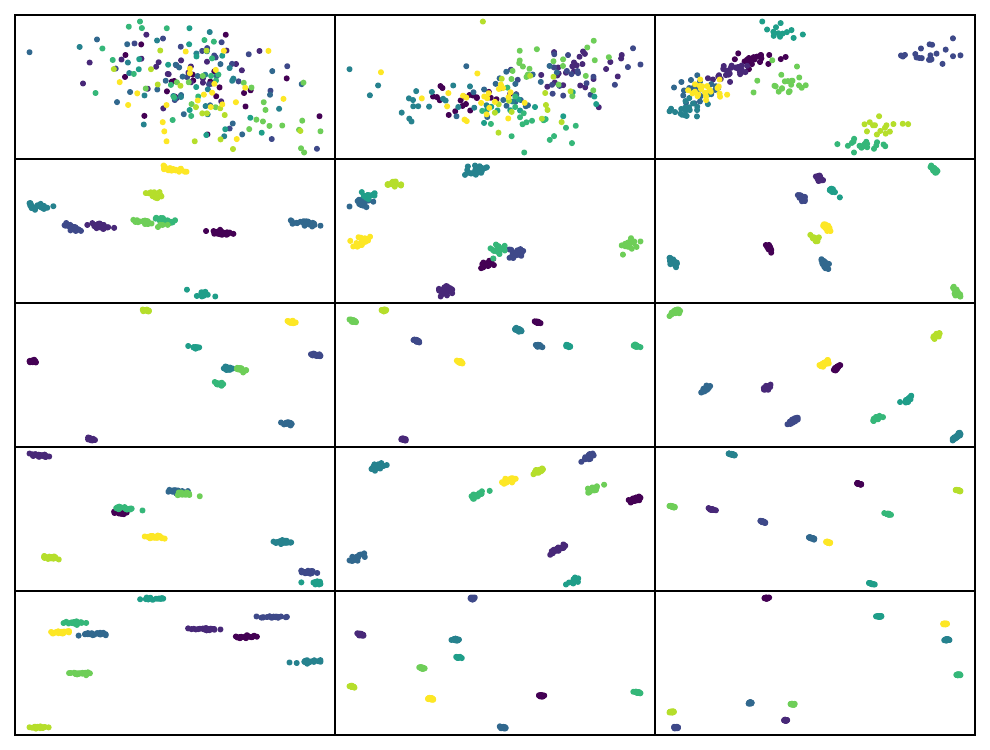} \\
  Epoch 700 & 
Epoch 1200 \\
  \includegraphics[width=0.4\textwidth]{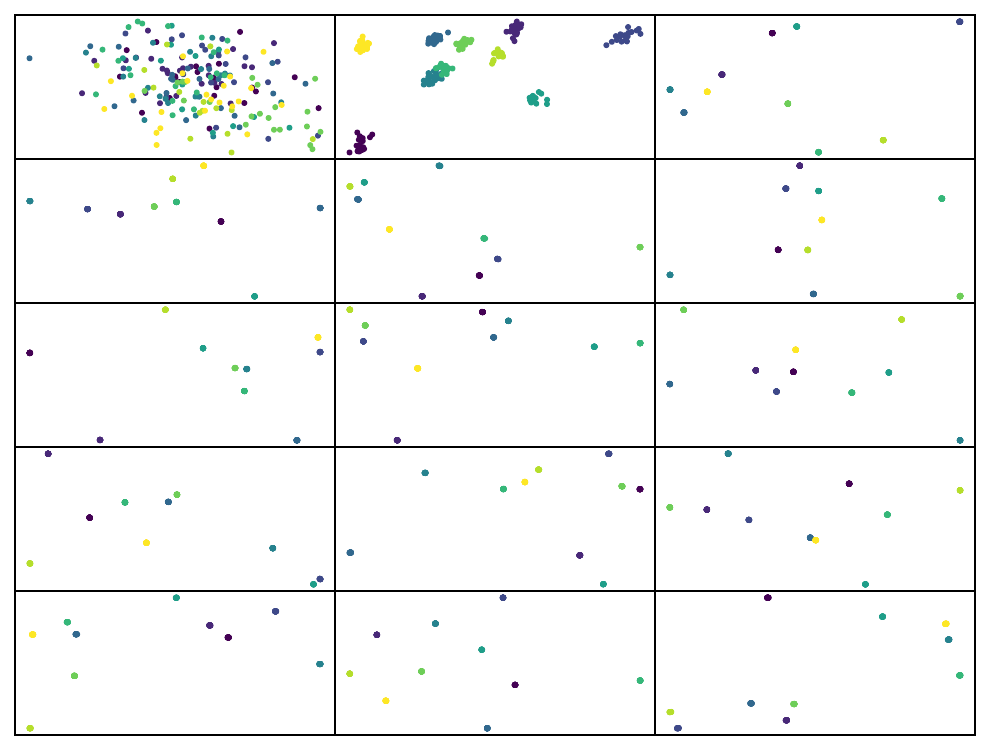} &   \includegraphics[width=0.4\textwidth]{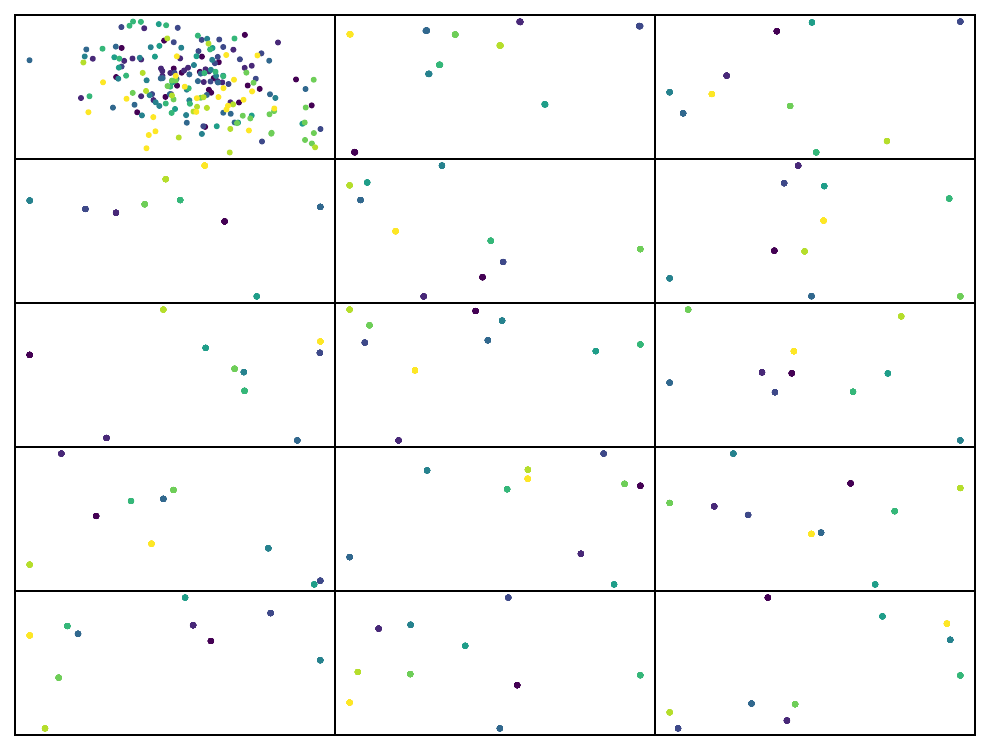} \\
\end{tabular}
\caption{Neural collapse of features in a random $d = 2$ dimensional subspace, where each color represents a class of MNIST. Layers for each epoch are ordered from left to right and from top to down.}
\label{fig:random_proj}
\end{figure}

\end{document}